\documentclass[letterpaper, 10 pt, journal, twoside]{IEEEtran}

\usepackage{siunitx}
\usepackage{relsize}
\usepackage{ifthen}
\usepackage[colorinlistoftodos]{todonotes}

\usepackage[caption=false]{subfig}

\usepackage{graphics} %
\usepackage{graphicx}
\usepackage{rotating}
\usepackage{color}
\usepackage{enumerate}
\usepackage[T1]{fontenc}
\usepackage{psfrag}
\usepackage{epsfig} %
\usepackage{booktabs}
\usepackage{graphicx,url}
\usepackage{multirow}
\usepackage{array}
\usepackage{latexsym}
\usepackage{amsfonts}
\usepackage{amsmath}
\usepackage{amssymb}
\usepackage{mathtools}
\usepackage{amsthm}
\usepackage{xstring}
\usepackage{algorithm} 
\usepackage[noEnd=true, indLines=true, commentColor=NavyBlue, endLComment=~, beginComment=//~]{algpseudocodex} 
\usepackage{xpatch}
\usepackage{multirow}
\usepackage{xcolor}
\usepackage[dvipsnames]{xcolor}
\usepackage{prettyref}
\usepackage{flexisym}
\usepackage{bigdelim}
\usepackage{listings}

\usepackage{enumitem}
\usepackage{xspace}
\usepackage{bm}
\graphicspath{{./figures/}}
\usepackage{tikz}
\usetikzlibrary{matrix,calc}
\usepackage[acronym]{glossaries}
\usepackage[bookmarks=true]{hyperref}
\usepackage[capitalize]{cleveref}
\usepackage{subcaption}
\usepackage{aliascnt}
\usepackage{threeparttable}
\usepackage[normalem]{ulem}

\usepackage{fancyhdr}

\newrefformat{prob}{Problem\,\ref{#1}}
\newrefformat{def}{Definition\,\ref{#1}}
\newrefformat{sec}{Section\,\ref{#1}}
\newrefformat{sub}{Section\,\ref{#1}}
\newrefformat{prop}{Proposition\,\ref{#1}}
\newrefformat{app}{Appendix\,\ref{#1}}
\newrefformat{alg}{Algorithm\,\ref{#1}}
\newrefformat{cor}{Corollary\,\ref{#1}}
\newrefformat{thm}{Theorem\,\ref{#1}}
\newrefformat{lem}{Lemma\,\ref{#1}}
\newrefformat{fig}{Fig.\,\ref{#1}}
\newrefformat{tab}{Table\,\ref{#1}}

\theoremstyle{definition}
\newtheorem{theorem}{Theorem}[section]

\theoremstyle{definition}
 
\crefname{problem}{Problem}{Problems}

\crefname{trule}{Rule}{Rules}

\crefname{corollary}{Corollary}{Corollaries}

\crefname{conjecture}{Conjecture}{Conjectures}

\newtheorem{lemma}{Lemma}[section]
\crefname{lemma}{Lemma}{Lemmas}

\crefname{assumption}{Assumption}{Assumptions}

\newtheorem{definition}{Definition}[section]
\crefname{definition}{Definition}{Definitions}

\newtheorem{proposition}{Proposition}[section]
\crefname{proposition}{Proposition}{Propositions}

\newtheorem{remark}{Remark}[section]
\crefname{remark}{Remark}{Remarks}

\crefname{example}{Example}{Examples}

\newcommand{\bdmath}{\begin{dmath}}
\newcommand{\edmath}{\end{dmath}}
\newcommand{\beq}{\begin{equation}}
\newcommand{\eeq}{\end{equation}}
\newcommand{\bdm}{\begin{displaymath}}
\newcommand{\edm}{\end{displaymath}}
\newcommand{\bea}{\begin{eqnarray}}
\newcommand{\eea}{\end{eqnarray}}
\newcommand{\beal}{\beq \begin{array}{ll}}
\newcommand{\eeal}{\end{array} \eeq}
\newcommand{\beas}{\begin{eqnarray*}}
\newcommand{\eeas}{\end{eqnarray*}}
\newcommand{\ba}{\begin{array}}
\newcommand{\ea}{\end{array}}
\newcommand{\bit}{\begin{itemize}}
\newcommand{\eit}{\end{itemize}}
\newcommand{\ben}{\begin{enumerate}}
\newcommand{\een}{\end{enumerate}}

\renewcommand{\boldsymbol}[1]{{\bm #1}}

\newcommand{\hide}[1]{}

\newcommand{\hiddenText}{{\color{gray} hidden text.}}
\newcommand{\hideWithText}[1]{\hiddenText}

\newcommand{\vtau}{\boldsymbol{\tau}}

\newcommand{\blue}[1]{{\color{blue}#1}}

\newcommand{\linkToPdf}[1]{\href{#1}{\blue{(pdf)}}}
\newcommand{\linkToPpt}[1]{\href{#1}{\blue{(ppt)}}}
\newcommand{\linkToCode}[1]{\href{#1}{\blue{(code)}}}
\newcommand{\linkToWeb}[1]{\href{#1}{\blue{(web)}}}
\newcommand{\linkToVideo}[1]{\href{#1}{\blue{(video)}}}
\newcommand{\linkToMedia}[1]{\href{#1}{\blue{(media)}}}
\newcommand{\award}[1]{\xspace} %

\newcommand{\myparagraph}[1]{\noindent\textbf{#1}}

\newcommand{\traj}{\tau}

\newcommand{\cI}{\mathcal{I}}

\newcommand{\bmat}{\left[ \begin{array}}
\newcommand{\emat}{\end{array}\right]}

\newcommand{\tup}[1]{\left( #1\right)}

\newacronym{acr:mapf}{MAPF}{Multi-Agent Path Finding}
\newacronym{acr:fico}{FICO}{Finite-Horizon Closed-Loop Factorization}
\newacronym{acr:soc}{SOC}{Sum of Cost}
\newacronym{acr:mpc}{MPC}{Model Predictive Control}
\newacronym{acr:rl}{RL}{Reinforcement Learning}
\newacronym{acr:ert}{ERT}{Execution Response Time}
\newacronym{acr:accbs}{ACCBS}{Anytime Closed-Loop Conflict-Based Search}
\newacronym{acr:cbs}{CBS}{Conflict-Based Search}

\usepackage{ifthen}
\newboolean{anonymous}
\setboolean{anonymous}{false}
\usepackage{changes}
\usepackage{cite}

\begin{document}
\title{
Adaptive-Horizon Conflict-Based Search for \\ Closed-Loop Multi-Agent Path Finding
}

\author{Jiarui Li$^{1}$,
        Federico Pecora$^{2}$,
        Runyu Zhang$^{1}$,
        and Gioele Zardini$^{1}$
\thanks{Manuscript received: June 11, 2026; Accepted June 23, 2026.}%
\thanks{This paper was recommended for publication by Editor Chao-Bo Yan upon evaluation of the Associate Editor and Reviewers' comments.
This work was supported by Prof. Zardini's grant from the MIT Amazon Science Hub, hosted in the Schwarzman College of Computing. Zhang was supported by the MIT Postdoctoral Fellowship
Program for Engineering Excellence.} %
\thanks{$^{1}$Jiarui Li, Runyu Zhang, and Gioele Zardini are with the Laboratory for Information and Decision Systems, Massachusetts Institute of Technology, Cambridge, MA, USA (e-mails: \{jiarui01, runyuzha, gzardini\}@mit.edu).}
\thanks{$^{2}$Federico Pecora is an independent contributor (e-mail: federico.pecora@gmail.com).}
}

\markboth{IEEE Robotics and Automation Letters. Preprint Version. Accepted June, 2026}
{LI \MakeLowercase{\textit{et al.}}: Adaptive-Horizon Conflict-Based Search for Closed-Loop Multi-Agent Path Finding} 

\maketitle

\begin{abstract}
\gls{acr:mapf} is a core coordination problem for large robot fleets in automated warehouses and logistics.
Existing approaches are typically either open-loop planners, {which must compute complete trajectories before execution and therefore may incur substantial planning latency before actions can be taken}, or closed-loop heuristics without reliable performance guarantees, limiting their use in safety-critical deployments.
This paper presents \gls{acr:accbs}, a closed-loop algorithm built on a finite-horizon variant of \gls{acr:cbs} with a horizon-changing mechanism inspired by iterative horizon-deepening in \gls{acr:mpc}. \gls{acr:accbs} dynamically adjusts the planning horizon based on the available computational budget, and reuses a single constraint tree to enable seamless transitions between horizons. 
As a result, it produces high-quality feasible solutions quickly while being asymptotically optimal as the budget increases, exhibiting anytime behavior. 
Extensive case studies demonstrate that \gls{acr:accbs} achieves a favorable balance between computational efficiency, solution quality, and execution flexibility, while naturally accommodating online disturbances through its closed-loop formulation.
\end{abstract}


\IEEEpeerreviewmaketitle

\section{Introduction}

\IEEEPARstart{A}{utomated} warehouses and fulfillment centers, where thousands of robots sort and deliver items, are now central to global supply chains, e-commerce, and manufacturing. 
{This setting naturally gives rise to the \acrfull{acr:mapf} problem: moving agents from starts to goals on a graph while avoiding collisions~\cite{stern2019mapf}.}
\gls{acr:mapf} captures a core combinatorial difficulty of large-scale coordination and has attracted sustained interest across robotics and AI~\cite{d2012guest,shaoul2024accelerating,wang2025lns2+,paul2022multi,wurman2008coordinating,standley2010finding}.
{Despite its importance, optimal \gls{acr:mapf} is NP-hard~\cite{yu2013planning}.}
{Existing solvers broadly trade scalability for guarantees: heuristic and approximate methods scale well but offer limited quality guarantees, whereas guaranteed solvers retain optimality or bounded-suboptimality guarantees but suffer worst-case exponential scaling~\cite{felner2017search}; see~\cref{sec:related-work}.}

\begin{figure}[tb]
    \centering
    \includegraphics[width=.8\linewidth, trim=1cm 1.5cm 1cm 1.5cm, clip]{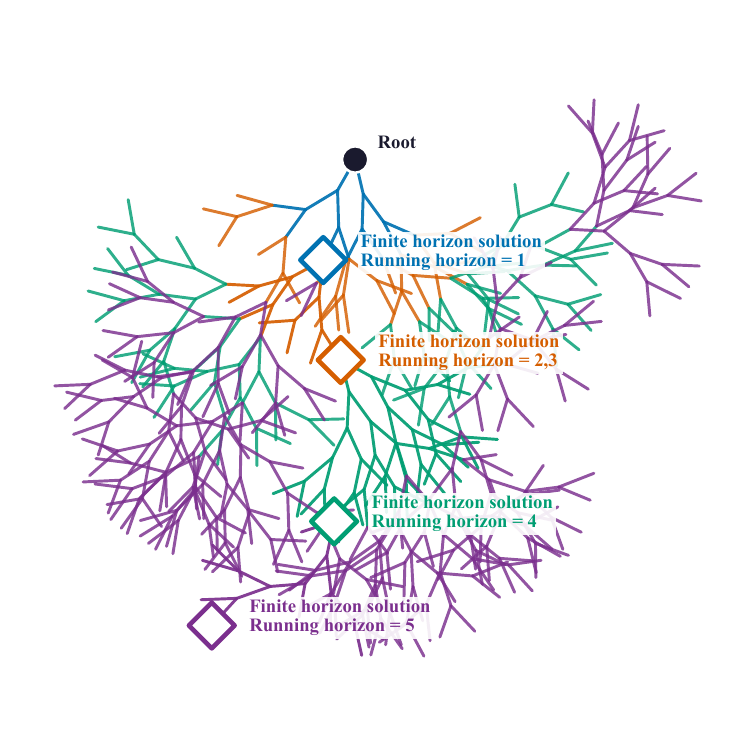}
    \caption{\textbf{Constraint tree in \gls{acr:accbs}.} Colors indicate successive stages of the search. 
    The same constraint tree is reused as the planning horizon grows, enabling seamless transitions between different horizon lengths (\cref{prop:tree-reuse}).}
    \label{fig:f1}
    \vspace*{-3mm}
\end{figure}

{
We build on \acrfull{acr:cbs}~\cite{sharon2015conflict}, a core guaranteed optimizer that best-first searches a binary \emph{constraint tree} of constraints, feasible per-agent trajectories, and joint costs, branching on conflicts until the first conflict-free, hence optimal, node is found.
However, \gls{acr:cbs} and variants such as ICBS~\cite{boyarski2015icbs}, ECBS~\cite{barer2014suboptimal}, and EECBS~\cite{li2021eecbs} are typically \emph{open-loop, full-trajectory} planners: execution starts only after complete trajectories are returned~\cite{atzmon2018robust}.
On dense, long-horizon instances, early conflict resolutions can induce later ones, producing large constraint trees, poor scalability, and long execution response times.
}
{
Closed-loop \gls{acr:mapf} reduces this latency by replanning from the current fleet state and returning only the next movement. 
This makes finite-horizon \gls{acr:cbs} natural: resolve conflicts over a finite lookahead, execute the first step, and replan online.
Yet this raises two challenges: the horizon must balance lookahead quality against exponential search growth, and finite-horizon \gls{acr:cbs} no longer inherits the standard completeness and optimality arguments; in particular, its cost is not directly a Lyapunov function for closed-loop operation~\cite{mayne2000constrained,bachtiar2016continuity}. 
Thus, naive receding-horizon \gls{acr:cbs} may lose guarantees.
}

{
We therefore seek a \emph{closed-loop} \gls{acr:mapf} algorithm that avoids costly infinite-horizon \gls{acr:cbs}-style replanning while retaining guarantees. 
Inspired by \gls{acr:mpc}, we use finite-horizon \gls{acr:cbs} as the core optimizer and add a dynamic horizon mechanism that expands lookahead over time, yielding finite-horizon efficiency with asymptotic optimality.
 {The dynamic horizon resolves both challenges at once: it replaces the brittle a priori horizon choice with cheap online horizon growth (\cref{sec:active-prefix-tree-reuse}), and it recovers the asymptotic optimality that a fixed window loses~\cite{li2021lifelong,veerapaneni2025windowed}.}
}

\paragraph*{Statement of Contribution}
This paper introduces \acrfull{acr:accbs}, an adaptive horizon \gls{acr:cbs} algorithm for solving closed-loop MAPF problems.
We first formalize a finite-horizon variant of \gls{acr:cbs} that enforces conflict-freedom only over a finite lookahead and is used in a closed-loop fashion. 
Building on this, we design \gls{acr:accbs}, which adapts the planning horizon online to the available computational budget via an \emph{active prefix} and constraint-tree reuse (\cref{fig:f1}), yielding an anytime algorithm that transitions seamlessly between horizons. 
We prove a cost invariance property (\cref{lem:cost-invariance}) that justifies this horizon adaptation, and establish completeness and asymptotic optimality  {(\cref{thm:accbs-theory})}, preserving the theoretical guarantees of \gls{acr:cbs}. 
Extensive experiments (\cref{sec:experiments}) show that, for a given time budget, \gls{acr:accbs} scales to larger instances than infinite-horizon \gls{acr:cbs} and its variants while maintaining high solution quality. 
Concretely, the dynamic horizon brings three merits.
\emph{(i)~It removes a brittle design choice}: the right fixed horizon depends jointly on the instance and the per-step budget (\cref{tab:horizon}), whereas \gls{acr:accbs} adapts it online and requires no such tuning.
\emph{(ii)~Its growth is provably cheap}: by cost invariance (\cref{lem:cost-invariance}) and constraint-tree reuse (\cref{prop:tree-reuse}), enlarging the horizon reuses the entire search tree, so the cumulative work is comparable to a \emph{single} finite-horizon \gls{acr:cbs} run at the largest horizon---this mechanism is the technical core of the paper.
\emph{(iii)~It restores guarantees that a fixed window loses}~\cite{li2021lifelong,veerapaneni2025windowed}, recovering \gls{acr:cbs}-style asymptotic optimality (\cref{thm:accbs-theory}).

\section{Related Work, Definitions, and Terminology} 
{
\subsection{Related work} \label{sec:related-work}
MAPF solvers broadly split into scalable heuristic methods and guaranteed methods. 
The former include CA*~\cite{silver2005cooperative}, OD~\cite{standley2010finding}, MAPF-LNS2~\cite{li2022mapf}, LaCAM~\cite{okumura2023lacam}, and anytime variants~\cite{li2021anytime,okumura2023improving,okumura2023engineering,gandotra2025anytime}; they scale well but usually lack optimality or bounded-suboptimality guarantees, which can be undesirable in safety-critical settings~\cite{shen2023tracking}. 
The latter include \gls{acr:cbs}~\cite{sharon2015conflict} and its variants ICBS~\cite{boyarski2015icbs}, ECBS~\cite{barer2014suboptimal}, and EECBS~\cite{li2021eecbs}, as well as BCP~\cite{lam2022branch}, ICTS~\cite{sharon2013increasing}, and M*~\cite{wagner2015subdimensional}; they return complete collision-free trajectories with certified quality, but are typically \emph{open-loop, full-trajectory} planners with worst-case exponential scaling~\cite{felner2017search}.

A complementary line couples planning and execution. 
RHCR~\cite{li2021lifelong} periodically replans over a fixed window for lifelong \gls{acr:mapf}~\cite{ma2017lifelong}, while the \gls{acr:mapf} System framework~\cite{li2025fico} formalizes feedback-based \emph{closed-loop} \gls{acr:mapf} algorithms that output one fleet movement from the current state at each step. 
This enables receding-horizon planning, but fixed-window methods require an a priori horizon, and existing closed-loop heuristics do not recover the CBS-style asymptotic guarantees pursued here.

\myparagraph{Positioning of \gls{acr:accbs}} --
\gls{acr:accbs} combines finite-horizon closed-loop execution with guaranteed \gls{acr:cbs}-style search. 
Unlike fixed-window or prior closed-loop methods \cite{li2021lifelong,li2025fico}, it adapts the horizon online via iterative deepening and reuses one constraint tree across horizons, avoiding brittle horizon selection while recovering asymptotic optimality.
}

\subsection{Definitions and terminology} \label{sec:terminologies}

\begin{definition}[\gls{acr:mapf} instance]
\label{def:mapf-inst}
A \gls{acr:mapf} \emph{instance} is a tuple~$\cI=\tup{G,A,\rho_s, \rho_g}$ where~$G=\tup{V,E}$ is a directed reflexive graph,~$A=\{a_1, \ldots, a_N\}$ is a set of~$N$ agents, and~$\rho_s,\rho_g\colon A\to V$ map each agent~$a\in A$ to its start vertex~$s_a=\rho_s(a)$ and goal vertex~$g_a=\rho_g(a)$.
\end{definition}

\begin{definition}[Trajectories and conflicts]
For an agent $a_i\in A$, a \emph{trajectory} is a finite sequence of vertices
$\traj^{a_i}=[v^i_0,v^i_1,\ldots,v^i_{T_i}]$.
Given trajectories $\traj^{a_i}$ and $\traj^{a_j}$ of equal length $T$, we define:
(a) a \emph{vertex conflict} if $\exists t\in\{0,\ldots,T\}$ such that $v^i_t = v^j_t$; and
(b) an \emph{edge conflict} if $\exists t\in\{0,\ldots,T\}$ such that $v^i_t = v^j_{t-1}$ and $v^i_{t-1} = v^j_t$.
Two trajectories are \emph{conflict-free} if they exhibit neither vertex nor edge conflicts.
\end{definition}

We denote by~$\vtau^A$ the collection of trajectories for all agents in~$A$\footnote{
Since~$G$ is reflexive, an agent that reaches its goal can wait there or temporarily vacate it. 
We therefore assume, without loss of generality, that all trajectories share a common length (the \emph{makespan}).
}. 
Recent work~\cite{li2025fico} introduced the notion of a \emph{\gls{acr:mapf} system}, a system-level framework that integrates planning and execution within a single feedback loop.
The resulting \emph{unified \gls{acr:mapf} problem} casts a range of \gls{acr:mapf} variants (one-shot~\cite{stern2019mapf}, lifelong~\cite{ma2017lifelong}, execution delay~\cite{ma2017multi}, etc.) as instances of a single controller-synthesis problem.

\begin{definition}[\gls{acr:mapf} system, unified \gls{acr:mapf} problem, and algorithms] \label{def:mapf-system}
At each discrete time $t\in\mathbb{N}_0$, let $\cI_t = \tup{G^t,A^t,\rho_s^t,\rho_g^t}$ be the current \gls{acr:mapf} instance, where $G^t=\tup{V^t,E^t}$. A \emph{state} is a map $x_t\in (V^t)^{A^t}$, where $x_t(a)$ is the position of agent $a\in A^t$ at time $t$. A joint \emph{movement command} is a map $u_t\in (E^t)^{A^t}$ satisfying~${\textstyle \forall a\in A^t:\; u_t(a) = \tup{x_t(a),\,x_{t+1}(a)} \in E^t}$,
for the executed next state $x_{t+1}$; waiting is modeled by self-loop edges $(v,v)\in E^t$.
The \emph{\gls{acr:mapf} system} models planning and execution in a feedback loop in which an \emph{Actuator} and \emph{Environment} capture uncertainties, and a \emph{Controller} selects joint movements. The \emph{unified \gls{acr:mapf} problem} is to design a controller
$g \colon (x_t,\cI_t,t) \mapsto \hat{u}_t$
that outputs a planned movement $\hat{u}_t$ at each time $t$ until a terminal condition is met. 
We call $g$ \emph{open-loop} if $\hat{u}_t = g(\cI_t,t)$, i.e., the entire trajectories are computed a priori, and \emph{closed-loop} if $\hat{u}_t = g(x_t,\cI_t,t)$, i.e., movements are recomputed online from the current state.
\end{definition}

\section{\gls{acr:accbs}}

In this section, we present \gls{acr:accbs}. 
We first formulate a finite-horizon {closed-loop} variant of \gls{acr:cbs} (\cref{sec:finite-horizon-cbs}) that optimizes over an~$H$-step prefix and returns only the first control input. 
We then discuss trade-offs induced by choices of~$H$ (\cref{sec:horizon-changing}), introduce \emph{active prefix} and constraint-tree reuse for changing horizons (\cref{sec:active-prefix-tree-reuse}), and present the full algorithm and its guarantees (\cref{sec:accbs-details}).

\subsection{Finite-horizon closed-loop CBS} \label{sec:finite-horizon-cbs}
The classical \gls{acr:cbs} algorithm~\cite{sharon2015conflict} and its variants~\cite{boyarski2015icbs,barer2014suboptimal,li2021eecbs} are typically used as \emph{open-loop} \gls{acr:mapf} solvers: all conflicts over the entire horizon are resolved \emph{before} execution, which can be computationally expensive.

We instead consider a finite-horizon closed-loop variant of \gls{acr:cbs} that optimizes only over the first~$H$ time steps, where $H$ is a fixed hyperparameter.
At each time~$t$, it plans~$H$-step trajectories that are conflict-free on~$\{t,\dots,t+H\}$ and returns only the \emph{first} movement as the control input.
Embedded in the \gls{acr:mapf} system (\cref{def:mapf-system}) and
invoked at every step, this yields a feedback (\emph{closed-loop})
controller.
\begin{definition}[$H$-step trajectory]
\label{def:finite-horizon-traj}
Fix a time~$t$ and the \gls{acr:mapf} instance~$\cI_t = \tup{G^t, A^t, \rho_s^t, \rho_g^t}$ with~$G^t = \tup{V^t, E^t}$, and let~$x_t$ be the current state.
For an agent~$a_i \in A^t$, an \emph{$H$-step trajectory} at time~$t$ is a sequence~$\tau^{a_i,H}_{\mathrm{finite}}
        = [v^{a_i}_0, v^{a_i}_1, \ldots, v^{a_i}_H]$ of vertices in~$V^t$ such that~$v^{a_i}_0 = x_t(a_i)$ and~$\tup{v^{a_i}_\ell, v^{a_i}_{\ell+1}} \in E^t$ for all~$\ell \in \{0,\dots,H-1\}$.
\end{definition}

\begin{definition}[Vertex and edge constraints]
Let~$H\in \mathbb{N}$ be the planning horizon.    
A \emph{vertex constraint} is a tuple $c^{\mathrm{vtx}}=\tup{a_i,t_i,v_i}$, where~$a_i\in A^t$, $v_i\in V^t$, and~$t_i\in\{0,\dots,H\}$, prohibiting agent $a_i$ from being at vertex $v_i$ at time $t_i$.
An \emph{edge constraint} is a tuple $c^{\mathrm{edg}}=\tup{a_i,t_i,e_i}$, where $a_i\in A^t$, $e_i=\tup{u_i,w_i}\in E^t$, and $t_i\in\{0,\dots,H-1\}$, prohibiting agent $a_i$ from traversing edge $e_i$ from time $t_i$ to $t_i+1$.
A \emph{constraint set} is a finite set $C=\{c^{\mathrm{vtx}}_1,\ldots,c^{\mathrm{vtx}}_p,c^{\mathrm{edg}}_1,\ldots,c^{\mathrm{edg}}_q\}$.
\end{definition}

\begin{definition}[Constraint satisfaction]
Let~$\vtau_{\mathrm{finite}}^{A^t,H}
    = \{\tau^{a_1,H}_{\mathrm{finite}},\ldots,
        \tau^{a_N,H}_{\mathrm{finite}}\}$ be a set of $H$-step trajectories
    at time~$t$.
We say that~$\vtau_{\mathrm{finite}}^{A^t,H}$ \emph{satisfies} a
constraint set~$C$ if every constraint~$c\in C$ is obeyed by the corresponding agent trajectory: for~$c=\tup{a_i,t_i,v_i}$ (vertex),~$v^{a_i}_{t_i}\neq v_i$; for~$c=\tup{a_i,t_i,\tup{u_i,w_i}}$ (edge),~$\tup{v^{a_i}_{t_i}, v^{a_i}_{t_i+1}} \neq \tup{u_i,w_i}$.
\end{definition}

To make finite-horizon planning into a proxy for the original long-horizon objective, we equip it with an \gls{acr:mpc}-inspired cost that combines a running cost over the horizon with a terminal cost approximating the remaining time-to-go.

\begin{definition}[Cost function] 
    \label{def:cost-function}
Let~$\tau_{\mathrm{finite}}^{a_i,H}
    = [v^{a_i}_0, \ldots, v^{a_i}_H]$ be an~$H$-step trajectory of agent~$a_i$ at time~$t$.

    \noindent \textbf{Running cost.}
    For each time step~$\ell\in \{0,\ldots, H-1\}$, define
    \begin{equation*}
        p^{a_i}(v^{a_i}_\ell)
        =
        \begin{cases}
            1, & v^{a_i}_\ell \neq \rho_g^t(a_i), \\
            0, & \text{otherwise}.
        \end{cases}
    \end{equation*}
    Thus, the running cost penalizes each time step in which~$a_i$ has not yet reached its goal.

    \noindent \textbf{Terminal cost.}
    Let~$\gamma(v,a_i)$ be the length of a shortest path in~$G^t$ from~$v \in V^t$ to the goal~$\rho_g^t(a_i)$.
    The terminal cost is~$q^{a_i}(v^{a_i}_H) = \gamma(v^{a_i}_H, a_i)$, which approximates the remaining cost-to-go beyond the~$H$-step horizon.

    The total cost of a single trajectory is the sum of the running and terminal costs, and the cost of a set of trajectories is the sum over agents:~{$\mathcal{C}({\traj^{a_i,H}_{\mathrm{finite}}}) = \textstyle \sum_{t=0}^{H-1} p^{a_i}(v^{a_i}_t) + q^{a_i}(v^{a_i}_H),$~$\mathcal{C}({\vtau^{A^t,H}_{\mathrm{finite}}}) = \textstyle \sum_{a_i} \mathcal{C}(\traj^{a_i,H}_{\mathrm{finite}})$}.
\end{definition}

\begin{remark}[\gls{acr:mpc}-style formulation of the cost]
The cost in \cref{def:cost-function} has the standard \gls{acr:mpc} structure: a sum of \emph{stage costs}~$p^{a_i}$ over the horizon, plus a \emph{terminal cost}~$q^{a_i}$ that approximates the infinite-horizon tail.
Here,~$p^{a_i}$ coincides with the usual \gls{acr:mapf} sum-of-costs objective, while~$q^{a_i}$ is an admissible heuristic given by the graph distance to the goal.
This structure is crucial in \cref{sec:horizon-changing,sec:active-prefix-tree-reuse}, where we analyze cost behavior as the horizon~$H$ changes.
\end{remark}

\begin{definition}[Node in finite-horizon constraint tree]
\label{def:node}
Fix a horizon $H$ and a time $t$.
A \emph{node} in the finite-horizon constraint tree is a triple
$n = \bigl(C(n),\, \vtau^{A^t,H}_{\mathrm{finite}}(n),\, J_H(n)\bigr)$, where
$C(n)$ is a finite constraint set (vertex and edge constraints),
$\vtau^{A^t,H}_{\mathrm{finite}}(n)
        = \{\tau^{a_1,H}_{\mathrm{finite}},\ldots,
            \tau^{a_N,H}_{\mathrm{finite}}\}$ is a set of $H$-step trajectories at time $t$ for all agents $a_i\in A^t$ that satisfies $C(n)$, and~$J_H(n) = \mathcal{C}\bigl(\vtau^{A^t,H}_{\mathrm{finite}}(n)\bigr)$ is the \emph{node cost}.
\end{definition}

\begin{definition}[Finite-horizon constraint tree]
For a fixed horizon~$H$, the \emph{finite-horizon constraint tree}~$\mathcal{T}_H$ is a rooted tree whose nodes are as in \cref{def:node}. 
The root node has an empty constraint set and its trajectories are given by individual planning without constraints.
Each non-root node is obtained from its parent by adding a single vertex or edge constraint that resolves one chosen conflict.
During the \gls{acr:cbs} search, the set of unexpanded nodes of~$\mathcal{T}_H$ is maintained in a priority queue~$\mathrm{OPEN}$ sorted by~$J_H(n)$.
\end{definition}

{
\begin{remark}[Efficient low-level individual planning]
    \label{rem:lowlevel}
    In the finite-horizon closed-loop \gls{acr:cbs} in~\cref{sec:finite-horizon-cbs} and in \gls{acr:accbs} in~\cref{sec:accbs-details}, \textsc{IndPlan} solves each constrained single-agent planning problem. 
    It first uses constrained A* over the constrained trajectory segment, then uses a perfect heuristic, i.e., the exact shortest-path distance to the goal, computed lazily by backward BFS from the goal. 
    This avoids unnecessary preprocessing~\cite{okumura2023improving,okumura2023lacam,li2025fico}. 
\end{remark}
}

Note that the~$H$-step trajectories stored in a node may contain conflicts.
{
The prefix search in \cref{alg:prefix-cbs} is used in both finite-horizon closed-loop \gls{acr:cbs} and \gls{acr:accbs} later in this letter. It expands the constraint tree until it finds a node whose trajectories are conflict-free over the first $H$ steps.
Finite-horizon closed-loop \gls{acr:cbs} applies it directly with a fixed horizon $H$ and returns the first movement of the returned node.
}

\makeatletter
\xpatchcmd{\algorithmic}{\itemsep\z@}{\itemsep=0.5pt}{}{}
\makeatother

\begin{algorithm}[t]
\begin{minipage}{\columnwidth}
    \small
    \caption{Prefix-\gls{acr:cbs} search primitive}
    \label{alg:prefix-cbs}
    \begin{algorithmic}[1]
    \Require Current instance $\cI_t$, state $x_t$, trajectory horizon $H$, active horizon $h$, queue $\mathrm{OPEN}$, time limit $t_{\max}$
    \While{$\mathrm{OPEN} \neq \mathsf{EmptyQueue}$ \textbf{and not} $\Call{TimeOut}{t_{\max}}$}
        \State $n \gets \mathrm{OPEN}.\mathbf{dequeue}()$
        \State $\mathrm{Conflict},a_i,a_j \gets \Call{FindConflict}{n.\vtau_{\mathrm{finite}}^{A^t,H},h}$
        \If{$\mathrm{Conflict}=\emptyset$}
            \State \Return $n$ \Comment{The active prefix of $n$ is conflict-free.}
        \EndIf
        \For{$a \in \{a_i,a_j\}$}
            \State $n_{\mathrm{new}}^{a}.\mathrm{Constr} \gets n.\mathrm{Constr} \cup \Call{NewConstr}{\mathrm{Conflict},a}$
            \State $n_{\mathrm{new}}^{a}.\vtau_{\mathrm{finite}}^{A^t,H} \gets \Call{IndPlan}{\cI_t,x_t,H,n_{\mathrm{new}}^{a}.\mathrm{Constr}}$
            \State $\mathrm{OPEN}.\mathbf{insert}(n_{\mathrm{new}}^{a})$
        \EndFor
    \EndWhile
    \State \Return $\emptyset$ \Comment{Timeout or no feasible prefix.}
    \end{algorithmic}
\end{minipage}
\end{algorithm}

{
\begin{remark}[Relation to Windowed \gls{acr:cbs}~\cite{li2021lifelong}]
\label{rem:comparison_windowed}
The finite-horizon \gls{acr:cbs} used here differs semantically from Windowed \gls{acr:cbs} (RHCR). 
RHCR is an \emph{open-loop} sequential generator for lifelong settings with assigned goal sequences: it commits chunks of $h \geq 1$ steps and stitches them into a global path. 
Our method is instead a \emph{memoryless closed-loop controller}: at each time~$t$, it replans from the current state~$x_t$, commits only the first step, and discards the rest. 
This absorbs immediate disturbances, such as execution delays or new agent arrivals, at every timestep and requires no known future goal sequence.
\end{remark}
}

\subsection{Horizon selection dilemma} \label{sec:horizon-changing}

In finite-horizon \gls{acr:cbs}, the horizon $H$ plays a role analogous to finite-horizon optimal control and \gls{acr:mpc}, determining how far ahead conflicts and costs are considered, with future effects captured via a terminal cost.
Increasing $H$ better approximates the long-horizon \gls{acr:mapf} problem, but also expands the search space and deepens the constraint tree $\mathcal{T}_H$.
As shown in \cref{tab:horizon}, small horizons are efficient but myopic, whereas large horizons improve solution quality at significantly higher computational cost.

\begin{table}[tb]
    \centering
    \footnotesize
    \caption{In finite-horizon \gls{acr:cbs}, the horizon choice strongly affects the computational cost, and finite horizons scale much better than infinite-horizon \gls{acr:cbs}.}
    \label{tab:horizon}
    \renewcommand{\arraystretch}{0.9}
    \setlength{\tabcolsep}{3.5pt}
    \begin{tabular}{@{}llccccccc@{}}
        \toprule
        & & & & \multicolumn{3}{c}{ {\textbf{Per-step time (ms)}}} & \\
        \cmidrule(lr){5-7}
        \textbf{Map} & \textbf{$N$} & \textbf{$H$} & \textbf{SOC} &  {\textbf{mean}} &  {\textbf{std}} &  {\textbf{$p_{95}$}} & \textbf{ERT (ms)$^\ddagger$} \\[0.3em]
        \midrule
        \multirow{12}{*}{\rotatebox[origin=c]{90}{\texttt{Small Empty}}}
            & \multirow{4}{*}{10}  & 1         & 15          &  {0.19}   &  {0.13}   &  {0.43}    & 0.441    \\
            &                      & 3         & 1           &  {0.23}   &  {0.17}   &  {0.50}    & 0.523    \\
            &                      & 5         & 1           &  {0.27}   &  {0.18}   &  {0.57}    & 0.584    \\
            &                      & $\infty^*$& 1           &  {--}     &  {--}     &  {--}      & 8.241    \\
        \cmidrule(l){2-8}
            & \multirow{4}{*}{20}  & 1         & 8           &  {0.20}   &  {0.15}   &  {0.40}    & 0.440    \\
            &                      & 3         & 5           &  {1.38}   &  {1.24}   &  {3.61}    & 3.806    \\
            &                      & 5         & 4           &  {5.72}   &  {5.49}   &  {15.33}   & 16.679   \\
            &                      & $\infty$  & 3           &  {--}     &  {--}     &  {--}      & 2291.28  \\
        \cmidrule(l){2-8}
            & \multirow{4}{*}{30}  & 1         & 55          &  {0.13}   &  {0.08}   &  {0.24}    & 0.287    \\
            &                      & 3         & 19          &  {73.4}   &  {71.8}   &  {231.5}   & 232.167  \\
            &                      & 5         & 12          &  {1058.1} &  {1054.6} &  {3112.5}  & 3322.275 \\
            &                      & $\infty$  & N/A$^\dagger$ &  {N/A}  &  {N/A}    &  {N/A}     & N/A      \\
        \midrule
        \multirow{12}{*}{\rotatebox[origin=c]{90}{\texttt{Small Random}}}
            & \multirow{4}{*}{25}  & 1         & 10          &  {0.06}   &  {0.04}   &  {0.12}    & 0.130    \\
            &                      & 3         & 3           &  {0.09}   &  {0.06}   &  {0.19}    & 0.207    \\
            &                      & 5         & 2           &  {0.14}   &  {0.09}   &  {0.30}    & 0.312    \\
            &                      & $\infty$  & 2           &  {--}     &  {--}     &  {--}      & 44.919   \\
        \cmidrule(l){2-8}
            & \multirow{4}{*}{50}  & 1         & 38          &  {0.23}   &  {0.14}   &  {0.49}    & 0.508    \\
            &                      & 3         & 24          &  {0.55}   &  {0.76}   &  {1.35}    & 1.434    \\
            &                      & 5         & 23          &  {4.31}   &  {3.95}   &  {11.33}   & 12.419   \\
            &                      & $\infty$  & N/A         &  {N/A}    &  {N/A}    &  {N/A}     & N/A      \\
        \cmidrule(l){2-8}
            & \multirow{4}{*}{100} & 1         & 77          &  {21.5}   &  {19.4}   &  {58.7}    & 62.136   \\
            &                      & 3         & 66          &  {437.2}  &  {436.9}  &  {1286.0}  & 1352.405 \\
            &                      & 5         & 45          &  {692.8}  &  {705.3}  &  {2051.2}  & 2134.605 \\
            &                      & $\infty$  & N/A         &  {N/A}    &  {N/A}    &  {N/A}     & N/A      \\
        \bottomrule
    \end{tabular}
    \parbox{\linewidth}{\footnotesize{$^*$ $\infty$ represents classical \gls{acr:cbs} (infinite horizon). $^\dagger$ N/A: infinite-horizon \gls{acr:cbs} timed out after 5 minutes. $^\ddagger$ ERT = Execution Response Time (computation time before execution begins).  {For infinite-horizon \gls{acr:cbs}, planning is a one-shot solve; per-step statistics do not apply (``--'').}}}
    \vspace*{-3mm}
\end{table}

In practice, the ``right'' horizon depends on both the instance (size, density, topology) and the available time budget per control cycle.
Sparse environments may be well served by small~$H$, while dense ones often require larger~$H$ to avoid short-sighted decisions.
Since computation time for a fixed~$H$ can still vary widely across instances, choosing a single horizon offline is brittle, and trying multiple horizons in parallel is impractical for time-critical robotics.

To address this, we introduce \gls{acr:accbs}, which wraps finite-horizon \gls{acr:cbs} in an \emph{iterative deepening} scheme~\cite{lavalle2006planning}.
Instead of fixing~$H$ in advance, \gls{acr:accbs} maintains a \emph{running horizon}~$h_r$: at each time~$t$, it first solves with a small~$h_r$ to obtain a low-latency feasible movement~$\hat{u}_t$, and, while time remains and~$h_r < H_{\mathrm{max}}$, it increases~$h_r$ and refines the solution.
The planner can be interrupted at any point and returns the best movement found so far, so the effective horizon is \emph{adapted online} to the available time, giving \gls{acr:accbs} an \emph{anytime} character.
Although increasing the horizon seems expensive, a key contribution of this work is to show that full recomputation is unnecessary: in \cref{sec:active-prefix-tree-reuse}, we formalize an \emph{active prefix} and prove cost invariance (\cref{lem:cost-invariance}) and constraint-tree reuse (\cref{prop:tree-reuse}), so the overall complexity is comparable to running finite-horizon \gls{acr:cbs} once with horizon $H_{\max}$. 
This horizon-changing mechanism is the core of \gls{acr:accbs}.

\subsection{Active prefix and constraint-tree reuse} \label{sec:active-prefix-tree-reuse}
We now formalize the horizon-changing mechanism.
At a fixed time~$t$, let~$H_{\max}\in \mathbb{N}$ be a nominal planning horizon, and let the \emph{running horizon}~$h_r\in \{1,\ldots,H_{\max}\}$ specify how many initial time steps must be conflict-free;
conflicts beyond~$h_r$ are \emph{temporarily} ignored.
The key idea is to grow~$h_r$ incrementally, while reusing the same
constraint tree and node costs.

\paragraph{Active prefix}
For notational simplicity, fix a time~$t$ and write~$A=A^t$.
Consider a joint~$H_{\max}$-step trajectory~$\vtau^{A,H_{\max}}= \{\tau^{a_1,H_{\max}}_{\mathrm{finite}},\ldots,
        \tau^{a_N,H_{\max}}_{\mathrm{finite}}\}$, where for each agent~$\tau^{a_i,H_{\max}}_{\mathrm{finite}}
 = [v^{a_i}_0,\ldots,v^{a_i}_{H_{\max}}]$.

\begin{definition}[Running horizon and active prefix]
The \emph{active prefix} of the running horizon~$h_r$ associated with~$\vtau^{A,H_{\max}}$ is the joint trajectory~$\textstyle{\vtau^{A,h_r \mid H_{\max}}_{\mathrm{finite}}
        = \{\tau^{a_1,h_r \mid H_{\max}}_{\mathrm{finite}},\ldots,
            \tau^{a_N,h_r \mid H_{\max}}_{\mathrm{finite}}\}}$, 
    where for each agent~$a_i$,~$\tau^{a_i,h_r \mid H_{\max}}_{\mathrm{finite}}
        = [v^{a_i}_0,\ldots,v^{a_i}_{h_r}].$
\end{definition}

The active prefix is the portion of the planned trajectories on which conflict-freedom is currently enforced (see \cref{fig:example-active-prefix}).

\begin{figure}
    \centering
    \includegraphics[width=0.65\linewidth, trim=24cm 12.5cm 24cm 12.5cm, clip]{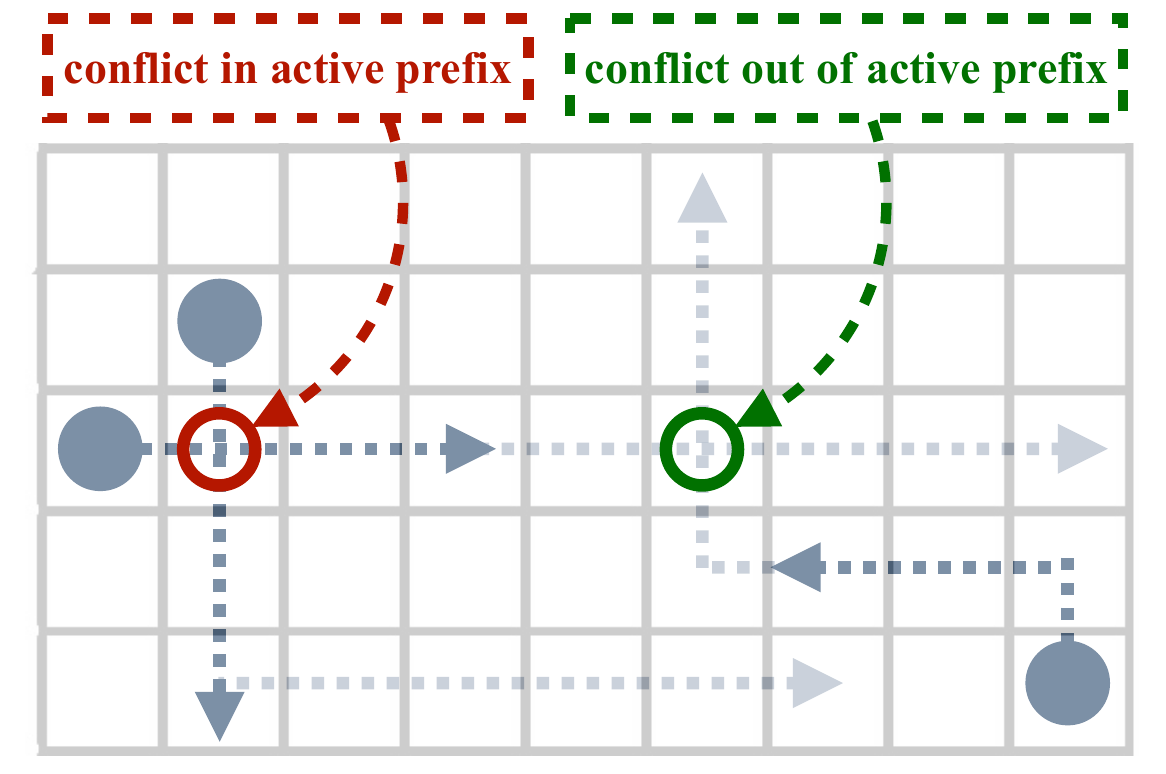}
    \caption{\textbf{Example of active prefix.} Only conflicts inside the active prefix (shaded region) are resolved; conflicts in the tail are ignored until the running horizon grows to cover them.}
    \label{fig:example-active-prefix}
    \vspace*{-3mm}
\end{figure}

In \gls{acr:accbs}, the low-level search at time~$t$ still operates on~$H_{\max}$-step trajectories as in finite-horizon \gls{acr:cbs}, but for a given running horizon~$h_r$ we only detect and resolve conflicts within the active prefix~$\vtau^{A,h_r \mid H_{\max}}_{\mathrm{finite}}$.
Once a node with a conflict-free active prefix is found, we extract the first-step control~$\hat{u}_t$ from that node. 
If time permits, we increase $h_r$ and continue the search, now resolving conflicts up to the new horizon. Growing~$h_r$ realizes the iterative-deepening behavior of \cref{sec:horizon-changing} without restarting the search.

The efficiency of this procedure hinges on two key properties:
i) node costs do not change when we increase~$h_r$, and ii) the constraint tree built for a smaller running horizon can be fully reused when~$h_r$ is enlarged.

\paragraph*{Cost invariance}
Recall that each node~$n$ (\cref{def:node}) stores a joint trajectory~$\vtau^{A,H_{\max}}_{\mathrm{finite}}(n)$ and a cost~$J_{H_{\max}}(n)$ defined by \cref{def:cost-function}.
For any~$h \in \{1,\dots,H_{\max}\}$ we also define the
\emph{prefix cost}~$\textstyle{
    J_h(n)
    = \sum_{a_i \in A}
      \left(
        \sum_{\ell=0}^{h-1} p^{a_i}(v^{a_i}_\ell)
        + \gamma\bigl(v^{a_i}_h, a_i\bigr)
      \right)},$
where~$[v^{a_i}_0,\ldots,v^{a_i}_{H_{\max}}]$ is the trajectory of
agent~$a_i$ stored in node~$n$.

{We assume the low-level planner has the following property}: for each node~$n$ and agent~$a_i$, beyond the largest time index appearing in~$C(n)$, the trajectory of~$a_i$ follows a shortest path to its goal (and then waits).
In particular, if~$n$ is created while the running horizon is~$h$, then its constraints only involve times~$\le h$, and the suffix~$[v^{a_i}_h,\ldots,v^{a_i}_{H_{\max}}]$ is a shortest path to~$\rho_g^t(a_i)$ (or stays at the goal).

\begin{lemma}[Cost invariance under horizon extension]
    \label{lem:cost-invariance}
Let~$n$ be a node generated while the running horizon is~$h$, and assume that for each agent~$a_i$ the suffix~$[v^{a_i}_h,\ldots,v^{a_i}_{H_{\max}}]$ follows a shortest path to~$\rho_g^t(a_i)$.
Then~$J_h(n) = J_{h+1}(n).$
Consequently,~$J_h(n) = J_{H_{\max}}(n)$ for all~$h \in \{1,\dots,H_{\max}\}$.
\end{lemma}

\begin{proof}
Consider a single agent~$a_i$ with trajectory~$\tau^{a_i,H_{\max}}_{\mathrm{finite}}
     = [v^{a_i}_0,\ldots,v^{a_i}_{H_{\max}}]$ in~$n$.
Its contributions to~$J_h(n)$ and~$J_{h+1}(n)$ are~{$\mathcal{C}_h^{a_i}
        = \textstyle\sum_{\ell=0}^{h-1} p^{a_i}(v^{a_i}_\ell)
           + \gamma(v^{a_i}_h,a_i),$~$\mathcal{C}_{h+1}^{a_i}
        = \textstyle\sum_{\ell=0}^{h} p^{a_i}(v^{a_i}_\ell)
           + \gamma(v^{a_i}_{h+1},a_i).$} 
Thus~$\mathcal{C}_{h+1}^{a_i} - \mathcal{C}_h^{a_i}
        = p^{a_i}(v^{a_i}_h)
          + \gamma(v^{a_i}_{h+1},a_i)
          - \gamma(v^{a_i}_h,a_i).$
By construction of the suffix, either  
(i)~$v^{a_i}_h \neq \rho_g^t(a_i)$, so~$p^{a_i}(v^{a_i}_h)=1$ and~$\gamma(v^{a_i}_h,a_i) = 1 + \gamma(v^{a_i}_{h+1},a_i)$, or  
(ii)~$v^{a_i}_h = \rho_g^t(a_i)$, so~$p^{a_i}(v^{a_i}_h)=0$ and~$\gamma(v^{a_i}_h,a_i) = \gamma(v^{a_i}_{h+1},a_i) = 0$.
In both cases the difference is 0, so~$\mathcal{C}_{h+1}^{a_i} = \mathcal{C}_h^{a_i}$.
Summing over all agents yields~$J_{h+1}(n)=J_h(n)$, and repeating this
argument gives~$J_h(n)=J_{H_{\max}}(n)$ for all~$h$.
\end{proof}

\cref{lem:cost-invariance} shows that enlarging the running horizon does not change the cost of nodes generated at smaller horizons: defining the node cost as~$J_{H_{\max}}(n)$ is consistent with any intermediate~$J_h(n)$ and therefore safe for tree reuse.

\paragraph*{Constraint-tree reuse}
Let~$\mathcal{T}_h$ be the constraint tree for running horizon~$h$: nodes are created by resolving conflicts within the active prefix of length~$h$, and their costs are~$J_h$.
Let~$\mathcal{T}_{h+1}$ be the tree for running horizon~$h+1$.

\begin{proposition}[Constraint-tree reuse]
    \label{prop:tree-reuse}
Fix~$h \in \{1,\dots,H_{\max}-1\}$ and assume the low-level planner extends trajectories beyond the largest constrained time index by following shortest paths to the goals, as in \cref{lem:cost-invariance}. 
Then:
a) Every node~$n \in \mathcal{T}_h$ is also a valid node in~$\mathcal{T}_{h+1}$ with the same constraint set and cost,~$C(n)$ unchanged, and~$J_{h+1}(n) = J_h(n)$.
b) The parent–child relations induced by resolving conflicts at times~$\le h$ are identical in~$\mathcal{T}_h$ and~$\mathcal{T}_{h+1}$, so~$\mathcal{T}_h$ is a rooted subtree of~$\mathcal{T}_{h+1}$.
\end{proposition}

\begin{proof}
By construction, nodes in~$\mathcal{T}_h$ are generated by repeatedly resolving conflicts that occur within the first $h$ time steps. Increasing the running horizon to~$h+1$ does not change these conflicts or the constraints used to resolve them, so the constraint set~$C(n)$ of every existing node~$n$ remains valid. 
New conflicts may appear only at time $h+1$, and these produce only new children of existing nodes.
By \cref{lem:cost-invariance}, for every node~$n$ generated while the running horizon was~$h$ we have~$J_{h+1}(n) = J_h(n)$. 
Thus all nodes of~$\mathcal{T}_h$ retain the same costs and remain admissible search nodes when the horizon is~$h+1$. 
Since edges in~$\mathcal{T}_h$ are created solely by resolving conflicts at times~$\le h$, and these conflicts are unchanged, the parent–child relations among nodes of~$\mathcal{T}_h$ are identical in~$\mathcal{T}_{h+1}$. Hence~$\mathcal{T}_h$ is a rooted subtree of~$\mathcal{T}_{h+1}$, with the same node costs, which proves the claim.
\end{proof}

\cref{prop:tree-reuse} shows that extending the running horizon preserves previous search effort: the search for~$h+1$ continues directly from~$\mathcal{T}_h$ without restart, so node expansions and low-level calls are comparable to a single finite-horizon \gls{acr:cbs} run with horizon~$H_{\max}$, while intermediate solutions for shorter horizons are exposed along the way (anytime behavior); \cref{sec:accbs-details} gives the operational details. 

\subsection{Details of \gls{acr:accbs} algorithm} \label{sec:accbs-details}

\makeatletter
\xpatchcmd{\algorithmic}{\itemsep\z@}{\itemsep=0.5pt}{}{}
\makeatother

\begin{algorithm}[t]
\resizebox{\columnwidth}{!}{
\begin{minipage}{\columnwidth}
    \small
    \caption{Anytime closed-loop \gls{acr:cbs} algorithm}
    \label{alg:accbs}
    \begin{algorithmic}[1]
    \Require Current instance $\cI_t$, state $x_t$, maximum horizon $H_{\max}$, time limit $t_{\max}$
    \State $h_r \gets 1$, $\mathrm{OPEN} \gets \mathsf{EmptyQueue}$
    \State $N_{\mathrm{root}}.\mathrm{Constr} \gets \emptyset$
    \State $N_{\mathrm{root}}.\vtau_{\mathrm{finite}}^{A^t,H_{\max}} \gets \Call{IndPlan}{\cI_t,x_t,H_{\max},N_{\mathrm{root}}.\mathrm{Constr}}$
    \State $\mathrm{OPEN}.\mathbf{insert}(N_{\mathrm{root}})$, $\vtau_{\mathrm{best}}^{A^t,H_{\max}} \gets \emptyset$
    \While{\textbf{not} $\Call{TimeOut}{t_{\max}}$ \textbf{and} $h_r \le H_{\max}$}
        \State $n \gets \Call{PrefixCBS}{\cI_t,x_t,H_{\max},h_r,\mathrm{OPEN},t_{\max}}$
        \If{$n=\emptyset$}
            \State \textbf{break}
        \EndIf
        \State $\vtau_{\mathrm{best}}^{A^t,H_{\max}} \gets n.\vtau_{\mathrm{finite}}^{A^t,H_{\max}}$
        \If{$h_r = H_{\max}$}
            \State \textbf{break}
        \EndIf
        \State $h_r \gets h_r+1$
        \State $\mathrm{OPEN}.\mathbf{insert}(n)$ \Comment{Reuse under larger active prefix.}
    \EndWhile
    \If{$\vtau_{\mathrm{best}}^{A^t,H_{\max}}=\emptyset$}
        \State \Return $\Call{PIBT}{\cI_t,x_t}$
    \EndIf
    \State \Return $\Call{ExtractFirstStepMovement}{\vtau_{\mathrm{best}}^{A^t,H_{\max}}}$
    \end{algorithmic}
\end{minipage}
}
\end{algorithm}

{
\myparagraph{Initialization and horizon adaptation} --
\gls{acr:accbs} starts with~$h_r=1$ and a root node containing unconstrained individually optimal~$H_{\max}$-step trajectories, ordered in~$\mathrm{OPEN}$ by~$J_{H_{\max}}$.
For the current~$h_r$, \textsc{PrefixCBS} performs best-first CBS expansion but detects conflicts only inside the active prefix.
If it returns a prefix-feasible node~$n$, the node becomes the incumbent; if~$h_r < H_{\max}$ and time remains, \gls{acr:accbs} increments~$h_r$, reinserts~$n$, and continues on the same tree.
By~\cref{lem:cost-invariance,prop:tree-reuse}, this horizon increase neither changes existing node costs nor invalidates existing parent--child relations.
The loop stops when a node is conflict-free for~$h_r=H_{\max}$ or the time budget expires; the first-step movement of the latest incumbent is returned, with PIBT used only if no $h_r=1$ incumbent is found.
}

\begin{remark}[Connection to PIBT and one-step optimization]
{When the $h_r=1$ prefix search completes,}
\gls{acr:accbs} returns a conflict-free next step that is \emph{globally optimal} w.r.t. the cost in \cref{def:cost-function}.
In contrast, PIBT~\cite{okumura2022priority} uses heuristic priority rules. Although Anytime PIBT optimizes the single-step transition (thus matching \gls{acr:accbs} at $h_r=1$), such one-step optimization is inherently myopic~\cite{gandotra2025anytime}. 
\gls{acr:accbs} mitigates this by iteratively extending the horizon, enabling multi-step coordination and substantially improved performance.
\end{remark}

\subsection{Discussion and theoretical analysis}
We analyze \gls{acr:accbs} in the one-shot \gls{acr:mapf} setting without uncertainty: a static instance~$\cI_0$, perfect actuation (executed and planned movements coincide), and a static environment.
The task terminates when all agents reach their goals and remain there, and performance is evaluated by the standard \gls{acr:soc}, often used for one-shot \gls{acr:mapf} and interpretable as energy consumption~\cite{surynek2016empirical}.

\begin{definition}[\gls{acr:soc}]
Let~$\vtau^{A}_{\mathrm{sol}} = \{\traj^{a_1}_{\mathrm{sol}},\ldots,\traj^{a_N}_{\mathrm{sol}}\}$ be solution trajectories of the one-shot \gls{acr:mapf} problem, where~$\traj_{\mathrm{sol}}^{a_i}=[v^{a_i}_0, \ldots, v^{a_i}_M]$ is the trajectory of agent~$a_i$. 
The \gls{acr:soc} is
    \[ \textstyle
        \mathrm{SOC}(\vtau^{A}_{\mathrm{sol}}) 
        = \sum_{a_i} \mathrm{SOC}(\traj_{\mathrm{sol}}^{a_i})
        = \sum_{a_i} \sum_t p^{a_i}(v^{a_i}_t),
    \]
    where~$p^{a_i}$ is the running cost from \cref{def:cost-function}. 
\end{definition}

As noted in \cref{def:cost-function}, if the planning horizon~$H_{\max}$ is at least the makespan of a solution, the finite-horizon cost~$J_{H_{\max}}$ coincides with the \gls{acr:soc}.

We consider an idealized regime in which the time limit~$t_{\max}$ in \cref{alg:accbs} is infinite, so the running horizon grows until it reaches~$H_{\max}$.

\begin{theorem}[Anytime behavior, completeness, and optimality of \gls{acr:accbs}]
\label{thm:accbs-theory}
Consider a one-shot \gls{acr:mapf} instance, and choose~$H_{\max}$ at least as large as the makespan of some \gls{acr:soc}-optimal solution. 
Run \gls{acr:accbs} with this~$H_{\max}$ and no time limit. 
Then:
(1) (\emph{Anytime property}) As runtime increases, the cost of the best solution stored in~$\vtau_{\mathrm{best}}^{A,H_{\max}}$ is 
{monotonically non-decreasing, bounded above by the \gls{acr:soc}-optimal cost, and converges to a finite limit}; at any time the algorithm can return the first-step movement of the best solution found so far.
(2) (\emph{Completeness \& optimality}) The algorithm terminates in finite time and returns a conflict-free joint trajectory of length~$H_{\max}$ that is globally optimal with respect to~$J_{H_{\max}}$, and therefore \gls{acr:soc}-optimal for the \gls{acr:mapf} instance under the stated~$H_{\max}$.
\end{theorem}

\begin{proof}
{For (2), once~$h_r=H_{\max}$, \gls{acr:accbs} coincides with standard \gls{acr:cbs} on horizon~$H_{\max}$ and inherits its completeness and optimality with respect to~$J_{H_{\max}}$. 
Since~$J_{H_{\max}}$ equals \gls{acr:soc} for all solutions with makespan at most~$H_{\max}$, the returned trajectory is \gls{acr:soc}-optimal.
For (1), fix a running horizon~$h_r$. The algorithm performs best-first search over a finite constraint tree, so the first conflict-free node updates~$\vtau_{\mathrm{best}}^{A,H_{\max}}$ with cost~$C^\star_{h_r}$, the horizon-$h_r$ optimum. 
As~$h_r$ grows, the active prefix expands, more conflicts must be resolved, and the feasible set shrinks; hence~$C^\star_{h_r}$ is non-decreasing. 
Moreover, any \gls{acr:soc}-optimal solution is feasible for every~$h_r \leq H_{\max}$, so~$C^\star_{h_r}$ is bounded above by the \gls{acr:soc}-optimal cost, equal to~$C^\star_{\infty}$. 
Thus the incumbent cost is monotone, bounded, and convergent, while its first-step movement can be returned at any time (anytime property).}
\end{proof}

\begin{figure}[t]
    \centering
    \subfloat[\href{https://movingai.com/benchmarks/mapf/empty-8-8.png}{\texttt{Small Empty Map}}]{\includegraphics[width=0.5\linewidth, trim=0cm 0.2cm 0cm 0.2cm, clip]{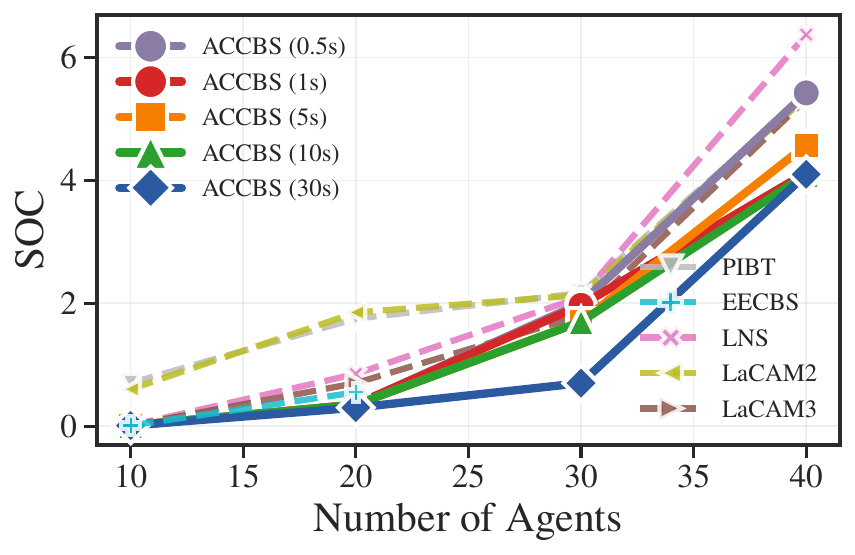}} \hfill
    \subfloat[\href{https://movingai.com/benchmarks/mapf/random-32-32-10.png}{\texttt{Small Random Map}}]{\includegraphics[width=0.5\linewidth, trim=0cm 0.2cm 0cm 0.2cm, clip]{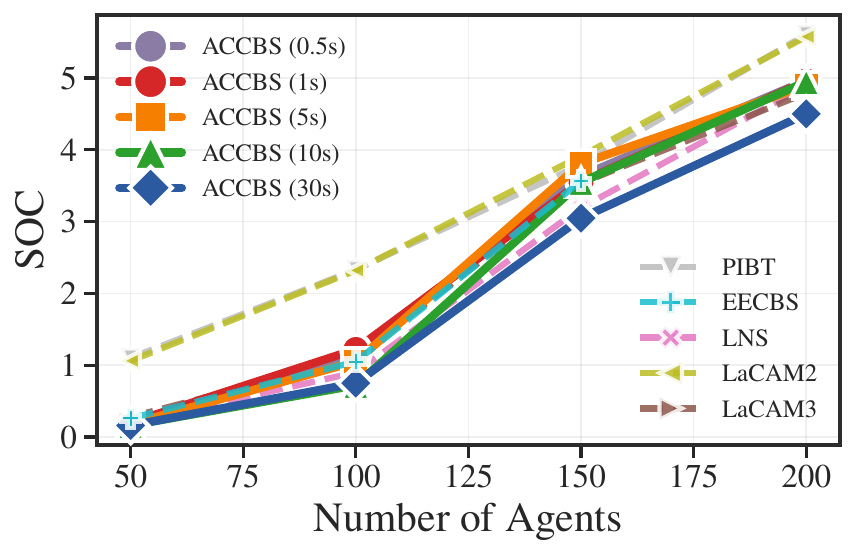}}
    \\
    \subfloat[\href{https://movingai.com/benchmarks/mapf/empty-48-48.png}{\texttt{Large Empty Map}}]{\includegraphics[width=0.5\linewidth, trim=0cm 0.2cm 0cm 0.2cm, clip]{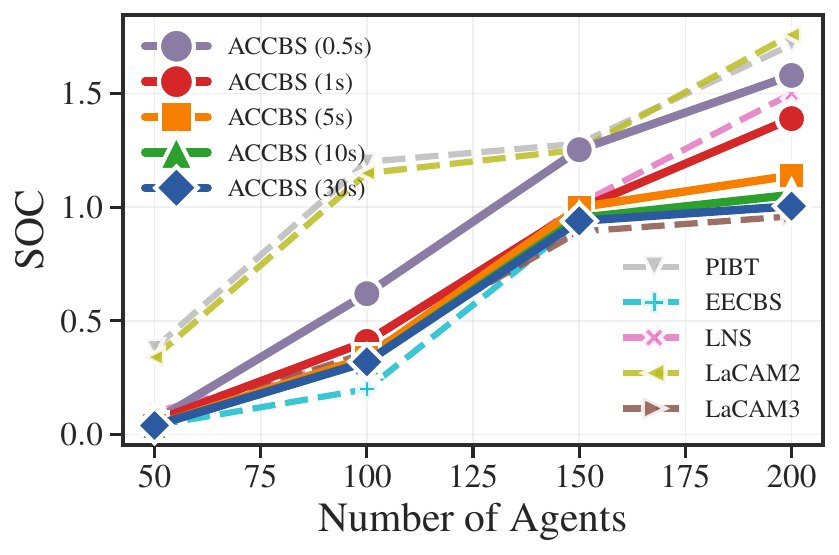}} \hfill
    \subfloat[\href{https://movingai.com/benchmarks/mapf/random-64-64-10.png}{\texttt{Large Random Map}}]{\includegraphics[width=0.5\linewidth, trim=0cm 0.2cm 0cm 0.2cm, clip]{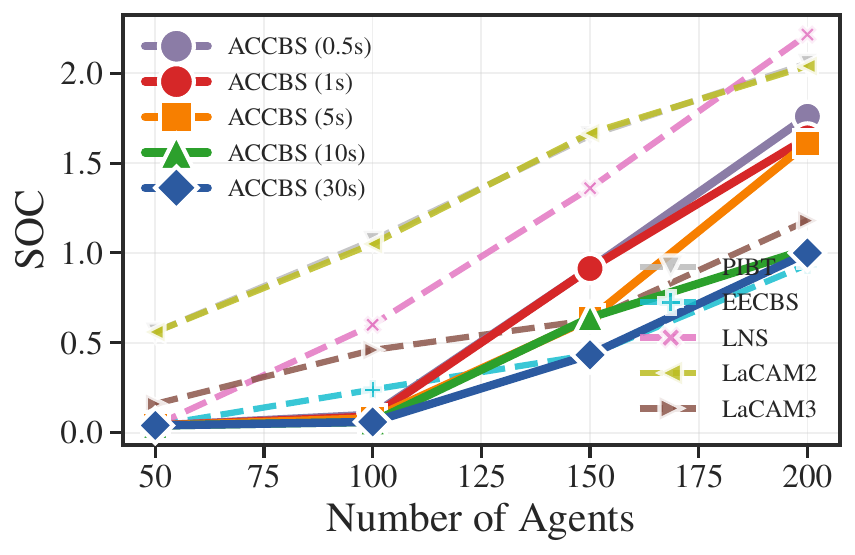}} 
    \\
    \subfloat[\href{https://movingai.com/benchmarks/mapf/index.html}{\texttt{Warehouse Map}}]{\includegraphics[width=0.5\linewidth, trim=0cm 0.2cm 0cm 0.2cm, clip]{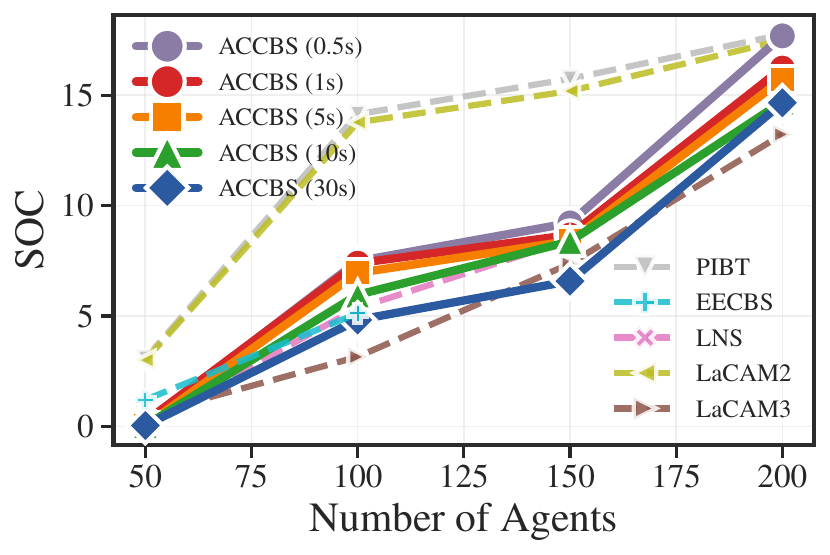}}
    \hfill
    \subfloat[\href{https://movingai.com/benchmarks/mapf/room-32-32-4.png}{\texttt{Room Map}}]{\includegraphics[width=0.5\linewidth, trim=0cm 0.2cm 0cm 0.2cm, clip]{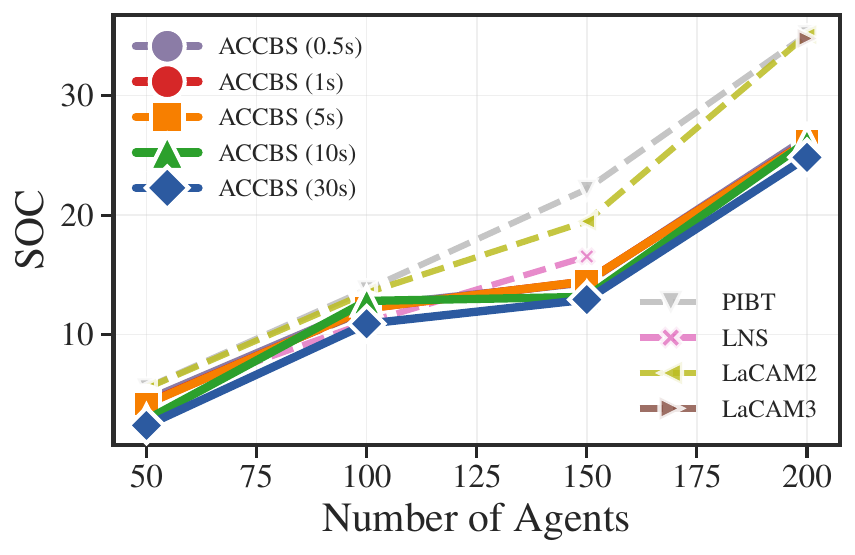}} 
    
    \caption{When the time budget increases, \gls{acr:accbs} can produce better \gls{acr:soc} performance and approach the optimum, which is very expensive to compute. {\gls{acr:accbs} yields competitive solution quality compared with multiple baselines, and its anytime design ensures a 100\% success rate by construction.
    }}
    \label{fig:one-shot-exp}
   \vspace*{-5mm}
\end{figure}

\begin{remark}[Single-step feasibility]
    \label{rem:recursive-feasibility}
     We claim feasibility in the following sense: at every timestep, a collision-free single-step command exists and is returned; 
    Although finite-horizon planning ignores conflicts beyond~$h_r$, \gls{acr:accbs} executes only the first step of a node whose active prefix, with~$h_r \geq 1$, is conflict-free, so the action is conflict-free by construction. 
    Such a node always exists: by \cref{def:mapf-inst} the graph is \emph{reflexive}, 
    the joint all-wait action is conflict-free and constitutes a conflict-free active prefix for any~$h_r \ge 1$. 
    Tail conflicts are reconsidered at the next step, when \gls{acr:accbs} replans from the updated state. 
    Applying this argument at each step ensures single-step feasibility throughout the closed-loop execution.
\end{remark}

\begin{remark}[Performance vs. time budget] 
\label{rem:perf-vs-budget}
As in \gls{acr:mpc}~\cite{bachtiar2016continuity}, strict monotonic improvement of trajectory quality with time budget (horizon) is not guaranteed for \gls{acr:accbs}. 
Nevertheless, increasing horizon reduces suboptimality in finite-horizon control~\cite{lavalle2006planning,grne2013nonlinear}: larger horizons mitigate myopia, letting the finite-horizon cost better approximate the true value function and typically yielding empirically better trajectories.
\end{remark}

\section{Experiments} \label{sec:experiments}
{We evaluate \gls{acr:accbs} in three settings of the unified \gls{acr:mapf} problem: one-shot planning, lifelong goal assignment, and stochastic execution uncertainties.}
All experiments use MAPF benchmark maps~\cite{stern2019mapf}: \href{https://movingai.com/benchmarks/mapf/empty-8-8.png}{\texttt{Small Empty Map}}, \href{https://movingai.com/benchmarks/mapf/random-32-32-10.png}{\texttt{Small Random Map}}, \href{https://movingai.com/benchmarks/mapf/empty-48-48.png}{\texttt{Large Empty Map}}, \href{https://movingai.com/benchmarks/mapf/random-64-64-10.png}{\texttt{Large Random Map}}, 
{\href{https://movingai.com/benchmarks/mapf/index.html}{\texttt{Warehouse Map}} and \href{https://movingai.com/benchmarks/mapf/room-32-32-4.png}{\texttt{Room Map}}}\footnote{All experiments run on a 2023 MacBook Pro (12-core CPU, 36 GB RAM)}.
{
\begin{remark}[Metrics, success rate, and runtime]
    \label{rem:metrics-selection}
    We report \gls{acr:soc} in one-shot \gls{acr:mapf}, where makespan is known to be easier than \gls{acr:soc}~\cite{surynek2016empirical}, and throughput in lifelong and uncertain-execution settings. 
    Since \gls{acr:accbs} starts from $h_r=1$, it returns a valid first-step movement in all evaluated regimes, yielding a $100\%$ success rate.
    As an \emph{anytime} algorithm, \gls{acr:accbs}' per-step planning time is capped at the assigned budget~$t_{\max}$.
\end{remark}
}

\begin{figure}[tb]
    \centering

    \subfloat[\href{https://movingai.com/benchmarks/mapf/empty-8-8.png}{\texttt{Small Empty Map}}]{\includegraphics[width=0.5\linewidth, trim=0cm 0.2cm 0cm 0.2cm, clip]{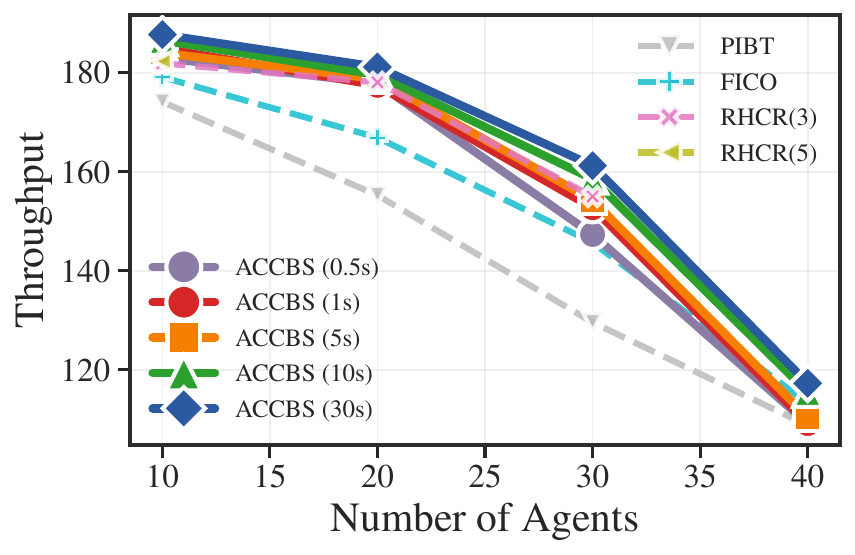}} \hfill
    \subfloat[\href{https://movingai.com/benchmarks/mapf/random-32-32-10.png}{\texttt{Small Random Map}}]{\includegraphics[width=0.5\linewidth, trim=0cm 0.2cm 0cm 0.2cm, clip]{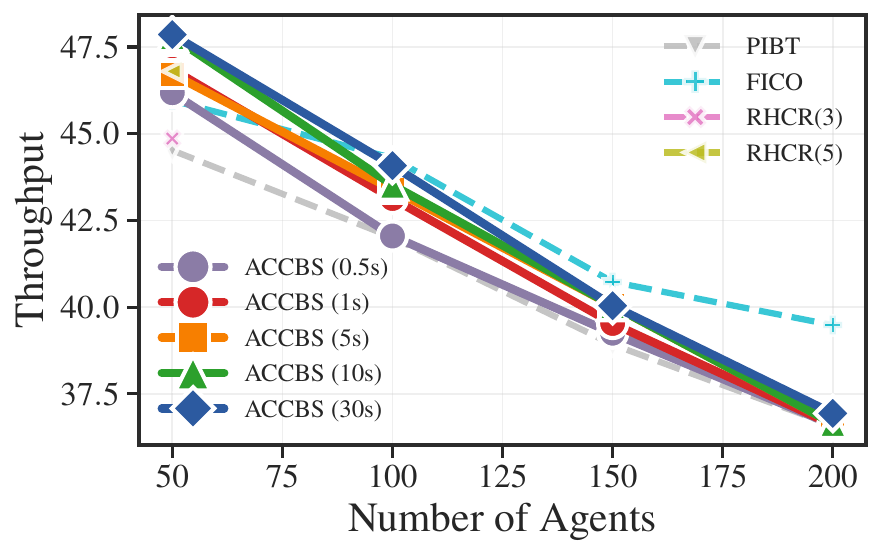}} 
    \\
    \subfloat[\href{https://movingai.com/benchmarks/mapf/empty-48-48.png}{\texttt{Large Empty Map}}]{\includegraphics[width=0.5\linewidth, trim=0cm 0.2cm 0cm 0.2cm, clip]{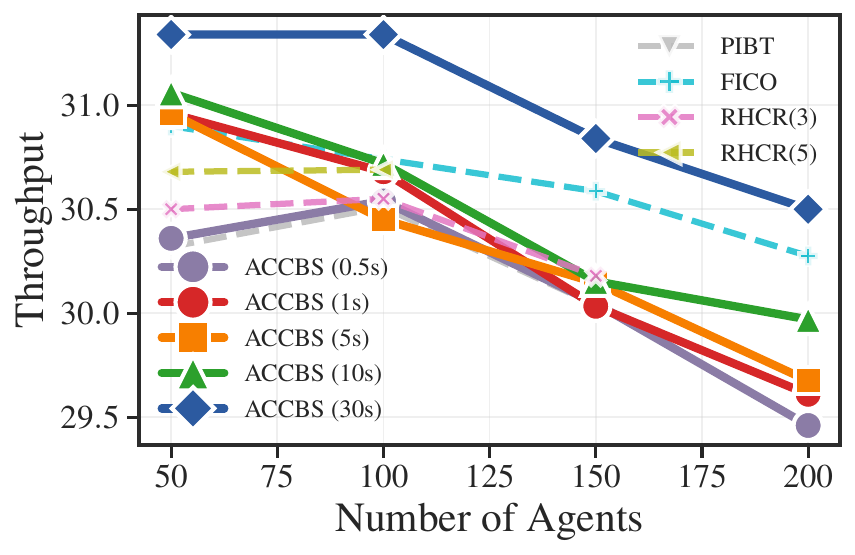}} \hfill
    \subfloat[\href{https://movingai.com/benchmarks/mapf/random-64-64-10.png}{\texttt{Large Random Map}}]{\includegraphics[width=0.5\linewidth, trim=0cm 0.2cm 0cm 0.2cm, clip]{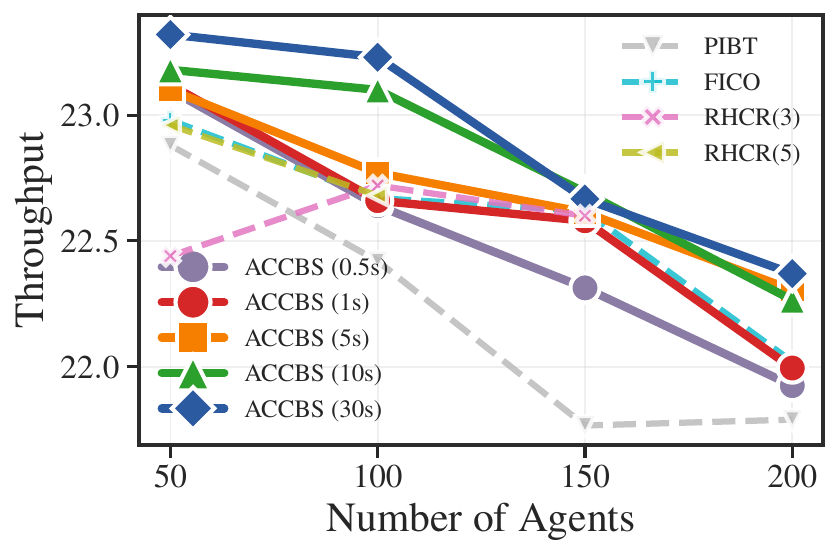}} 
    \\
    \subfloat[\href{https://movingai.com/benchmarks/mapf/index.html}{\texttt{Warehouse Map}}]{\includegraphics[width=0.5\linewidth, trim=0cm 0.2cm 0cm 0.2cm, clip]{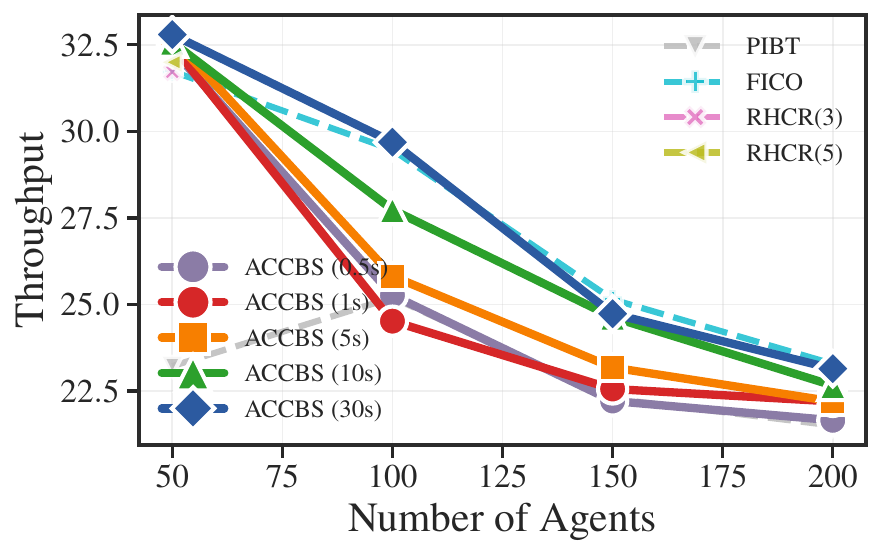}} \hfill
    \subfloat[\href{https://movingai.com/benchmarks/mapf/room-32-32-4.png}{\texttt{Room Map}}]{\includegraphics[width=0.5\linewidth, trim=0cm 0.2cm 0cm 0.2cm, clip]{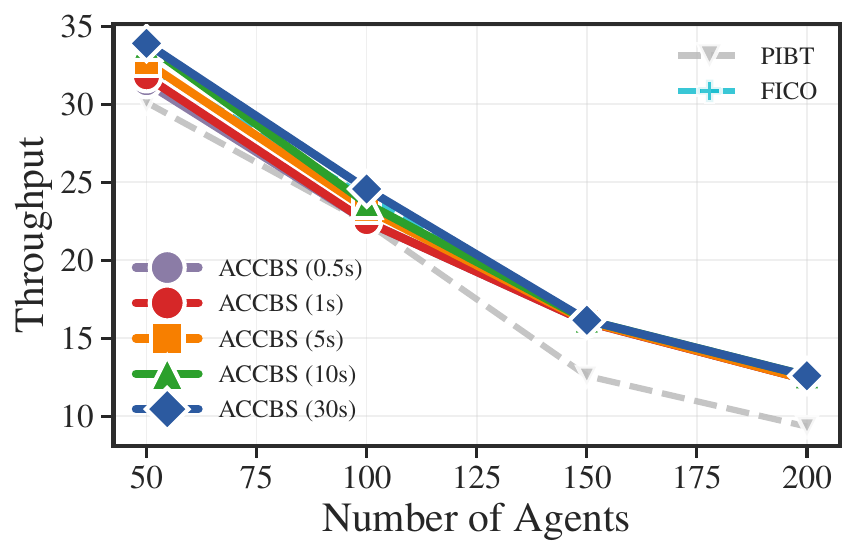}} 
    
    \caption{\gls{acr:accbs} can naturally deal with lifelong \gls{acr:mapf} with changing goals. Similar to one-shot case, more time budget can generally lead to better throughput performance. }
    \label{fig:lifelong-exp}
    \vspace*{-5mm}
\end{figure}

{\myparagraph{One-shot \gls{acr:mapf}} --
The classical one-shot setting uses a static instance, ideal actuator, and terminates when all agents reach their goals. 
\cref{fig:one-shot-exp} compares \gls{acr:accbs} with PIBT~\cite{okumura2022priority}, EECBS~\cite{li2021eecbs}, MAPF-LNS2~\cite{li2022mapf}, LaCAM2~\cite{okumura2023improving}, and LaCAM3~\cite{okumura2023engineering}.
As the per-step budget increases, \gls{acr:accbs} generally lowers \gls{acr:soc} and approaches optimality on easier instances. 
It improves over fast baselines such as PIBT and LaCAM2, remains competitive with stronger full-trajectory solvers, and avoids their failure mode of timing out before returning an executable plan.\footnote{As in the original experiments, the timeout is 1 minute. 
EECBS times out on the \texttt{Room Map} and larger \texttt{Warehouse Map}, \texttt{Small Random Map}, and \texttt{Small Empty Map}; MAPF-LNS2 on larger \texttt{Room Map}.}
Thus, \gls{acr:accbs} occupies the intended middle ground between expensive high-quality planning and fast lower-quality planning.
}

{
\myparagraph{Lifelong \gls{acr:mapf}} --
The closed-loop design makes \gls{acr:accbs} directly applicable to lifelong \gls{acr:mapf}, where agents receive new goals upon completing current ones. 
We run fixed-length episodes with an ideal actuator and report throughput, i.e., completed goals per agent over 1000 steps. 
\cref{fig:lifelong-exp} shows that larger budgets generally improve throughput, confirming that \gls{acr:accbs} turns additional computation into better online performance while remaining immediately executable.
Baselines include RHCR~\cite{li2021lifelong} with two window sizes, FICO~\cite{li2025fico}, and PIBT~\cite{okumura2022priority}.
\gls{acr:accbs} achieves competitive quality, and its adaptive horizon remains applicable when fixed-window RHCR times out.\footnote{As in the original paper, we use a 1-minute timeout per round, and RHCR uses \gls{acr:cbs} as its Windowed MAPF solver, matching the CBS-style search of \gls{acr:accbs}. RHCR consistently times out for larger fleets, including all \texttt{Room Map} experiments.
}
 {%
Each \gls{acr:accbs} curve in \cref{fig:lifelong-exp} corresponds to one budget ($\le 30$\,s per call), so comparing curves shows throughput as a function of per-step computation time. 
When a round times out, RHCR commits no steps, but the anytime design can always return an executable movement (\cref{rem:comparison_windowed}).}
}

\begin{figure}[t]
    \centering
    \subfloat[\href{https://movingai.com/benchmarks/mapf/empty-8-8.png}{\texttt{Small Empty Map}}, 20 agents, $p_{\mathrm{delay}} = p_{\mathrm{add}} = 0.1$]{\includegraphics[width=0.9\linewidth]{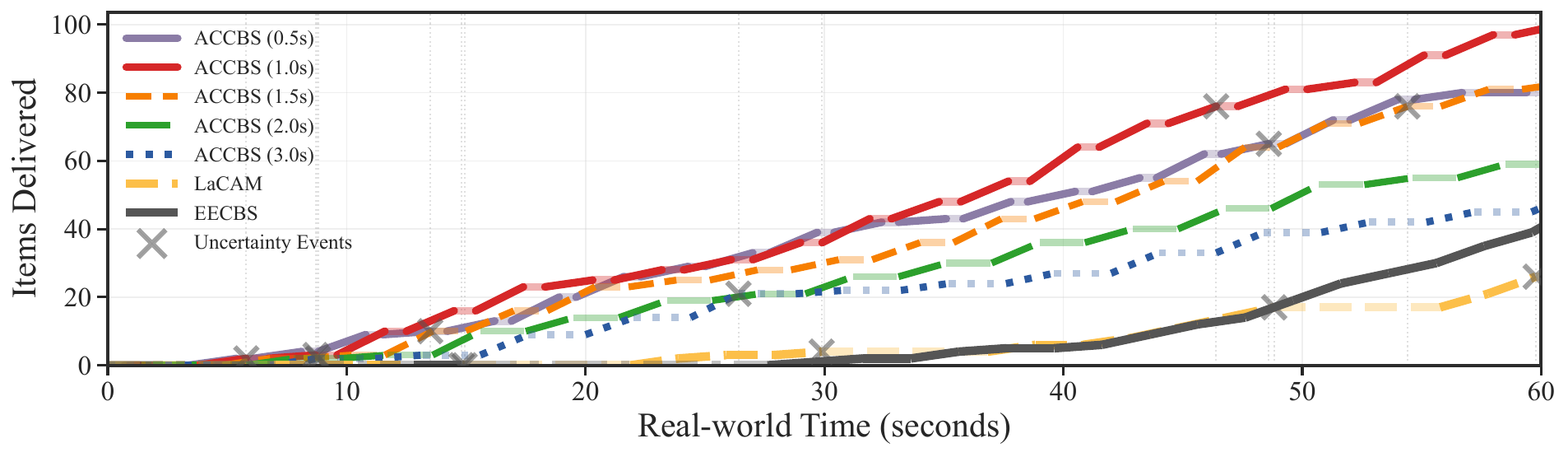}}
    \\
    \subfloat[\href{https://movingai.com/benchmarks/mapf/empty-8-8.png}{\texttt{Small Empty Map}}, 20 agents, $p_{\mathrm{delay}} = p_{\mathrm{add}} = 0.5$]{\includegraphics[width=0.9\linewidth]{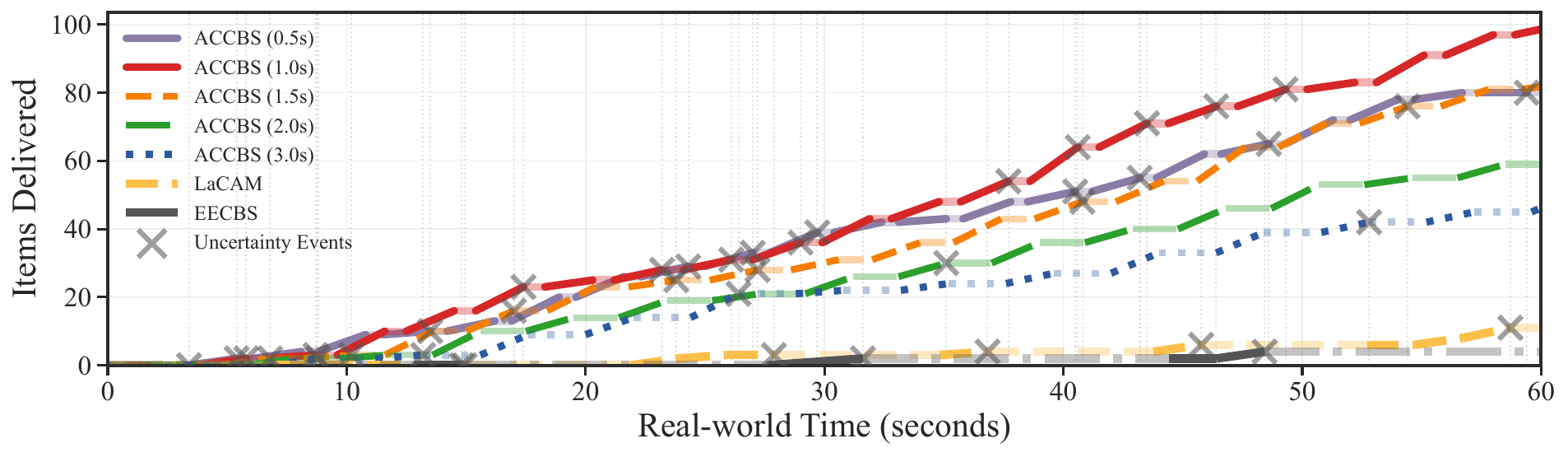}} 
    \caption{{
    This case study considers stochastic uncertainties, where open-loop algorithms are bottlenecked by frequent replanning 
    while closed-loop methods such as \gls{acr:accbs} excel. 
    Performance gains from longer ``thinking'' time may not justify the cost, making budget selection non-trivial. 
    }
    }
    \label{fig:uncertain-case-study}
    \vspace*{-5mm}
\end{figure}

{
\myparagraph{Realistic case study with uncertainties} --
We simulate mixed online events, with stochastic execution delays and agent arrivals/departures at each timestep. 
\gls{acr:accbs} absorbs these events at the next control update by replanning from the latest state and instance, whereas open-loop baselines must replan full trajectories whenever an event invalidates the current plan. 
\cref{fig:uncertain-case-study} shows that this replanning bottlenecks open-loop methods, while \gls{acr:accbs} maintains higher delivery throughput.
It also highlights the budget-selection tradeoff: longer ``thinking'' can improve decisions, but may reduce throughput when computation outweighs the quality gain.
}

{
\myparagraph{Prioritized-conflict ablation} --
Prioritized-conflict (PC) selection~\cite{boyarski2015icbs} can reduce tree expansions, but classifying conflicts also consumes the per-step budget. 
\cref{fig:remove-pc-ablation} shows that, in \gls{acr:accbs}, this overhead can outweigh the search-ordering benefit and reduce solution quality under fixed budgets.

}

\begin{figure}[tb]
    \centering
    \subfloat[\href{https://movingai.com/benchmarks/mapf/empty-48-48.png}{\texttt{Large Empty Map}}]{\includegraphics[width=0.5\linewidth, trim=0cm 0.2cm 0cm 0.2cm, clip]{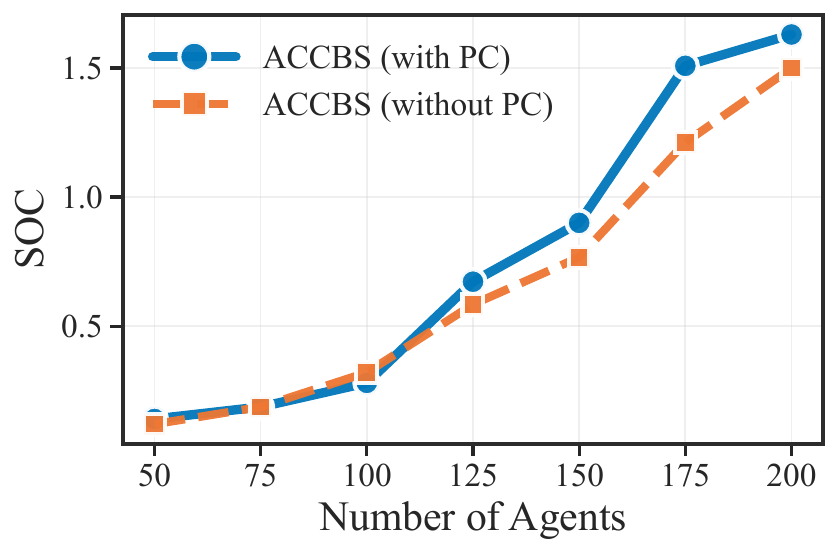}} \hfill
    \subfloat[\href{https://movingai.com/benchmarks/mapf/random-32-32-10.png}{\texttt{Small Random Map}}]{\includegraphics[width=0.5\linewidth, trim=0cm 0.2cm 0cm 0.2cm, clip]{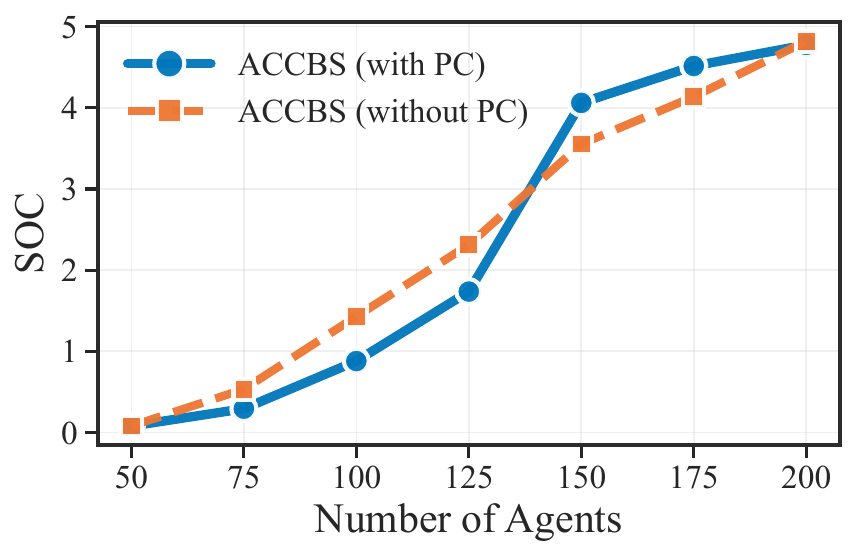}} 
    \caption{
    The conflict classification overhead may outweigh gains in tree exploration efficiency, lowering the solution quality. 
    }
    \label{fig:remove-pc-ablation}
    \vspace*{-5mm}
\end{figure}

\section{Conclusion and Future Work}

This paper introduced \gls{acr:accbs}, a closed-loop \gls{acr:mapf} algorithm built on a finite-horizon variant of \gls{acr:cbs} with a dynamically adjustable planning horizon driven by the available computational budget. 
By introducing the notion of an \emph{active prefix} and exploiting a cost invariance property, \gls{acr:accbs} reuses a single constraint tree across different horizons, enabling seamless horizon transitions and anytime behavior while retaining completeness and asymptotic optimality.

The flexibility of \gls{acr:accbs} and its connection to \gls{acr:mpc} open several avenues for future work. 
An immediate question is how to automatically tune the computational budget and horizon schedule for different applications. 
Other \gls{acr:cbs} variants, such as ECBS~\cite{barer2014suboptimal} and EECBS~\cite{li2021eecbs}, could also be embedded in the \gls{acr:accbs} framework to accelerate horizon expansion. 
 {Further directions include finer robustness metrics, comparisons with repair-based or partial-replanning execution strategies, and ablations isolating tree reuse from horizon adaptation.}
The anytime property naturally integrates with higher-level task scheduling, enabling holistic co-design of assignment and planning in large-scale systems.

{
  \bibliographystyle{IEEEtran}
  \bibliography{references}

@string{aaai = {Nat. Conf. on Artificial Intelligence (AAAI)}}

@string{ai = {Artificial Intelligence}}

@string{ijcai = {Intl. Joint Conf. on AI (IJCAI)}}

@string{springer = {Springer Verlag}}

@string{mpc = {Math. Program. Comput.}}

@article{okumura2022priority,
  title={Priority inheritance with backtracking for iterative multi-agent path finding},
  author={Okumura, Keisuke and Machida, Manao and D{\'e}fago, Xavier and Tamura, Yasumasa},
  journal=ai,
  volume={310},
  pages={103752},
  year={2022},
  publisher={Elsevier}
}

@inproceedings{okumura2023lacam,
  title={Lacam: Search-based algorithm for quick multi-agent pathfinding},
  author={Okumura, Keisuke},
  booktitle=aaai,
  volume={37},
  number={10},
  pages={11655--11662},
  year={2023}
}

@article{okumura2023improving,
  title={Improving lacam for scalable eventually optimal multi-agent pathfinding},
  author={Okumura, Keisuke},
  journal={arXiv preprint arXiv:2305.03632},
  year={2023}
}

@article{okumura2023engineering,
  title={Engineering LaCAM*: Towards Real-Time, Large-Scale, and Near-Optimal Multi-Agent Pathfinding},
  author={Okumura, Keisuke},
  journal={arXiv preprint arXiv:2308.04292},
  year={2023}
}

@article{wurman2008coordinating,
  title={Coordinating hundreds of cooperative, autonomous vehicles in warehouses},
  author={Wurman, Peter R and D'Andrea, Raffaello and Mountz, Mick},
  journal={AI magazine},
  volume={29},
  number={1},
  pages={9--9},
  year={2008}
}

@article{sharon2015conflict,
  title={Conflict-based search for optimal multi-agent pathfinding},
  author={Sharon, Guni and Stern, Roni and Felner, Ariel and Sturtevant, Nathan R},
  journal=ai,
  volume={219},
  pages={40--66},
  year={2015},
  publisher={Elsevier}
}

@inproceedings{li2021eecbs,
  title={Eecbs: A bounded-suboptimal search for multi-agent path finding},
  author={Li, Jiaoyang and Ruml, Wheeler and Koenig, Sven},
  booktitle=aaai,
  volume={35},
  number={14},
  pages={12353--12362},
  year={2021}
}

@article{wagner2015subdimensional,
  title={Subdimensional expansion for multirobot path planning},
  author={Wagner, Glenn and Choset, Howie},
  journal=ai,
  volume={219},
  pages={1--24},
  year={2015},
  publisher={Elsevier}
}

@inproceedings{li2022mapf,
  title={MAPF-LNS2: Fast repairing for multi-agent path finding via large neighborhood search},
  author={Li, Jiaoyang and Chen, Zhe and Harabor, Daniel and Stuckey, Peter J and Koenig, Sven},
  booktitle=aaai,
  volume={36},
  number={9},
  pages={10256--10265},
  year={2022}
}

@inproceedings{li2021anytime,
  title={Anytime multi-agent path finding via large neighborhood search},
  author={Li, Jiaoyang and Chen, Zhe and Harabor, Daniel and Stuckey, Peter J and Koenig, Sven},
  booktitle=ijcai,
  pages={4127--4135},
  year={2021}
}

@article{shen2023tracking,
  title={Tracking progress in multi-agent path finding},
  author={Shen, Bojie and Chen, Zhe and Cheema, Muhammad Aamir and Harabor, Daniel D and Stuckey, Peter J},
  journal={arXiv preprint arXiv:2305.08446},
  year={2023}
}

@inproceedings{standley2010finding,
  title={Finding optimal solutions to cooperative pathfinding problems},
  author={Standley, Trevor},
  booktitle=aaai,
  volume={24},
  number={1},
  pages={173--178},
  year={2010}
}

@article{stern2019mapf,
  title={Multi-Agent Pathfinding: Definitions, Variants, and Benchmarks},
  author={Roni Stern and Nathan R. Sturtevant and Ariel Felner and Sven Koenig and Hang Ma and Thayne T. Walker and Jiaoyang Li and Dor Atzmon and Liron Cohen and T. K. Satish Kumar and Eli Boyarski and Roman Bartak},
  journal={Symposium on Combinatorial Search (SoCS)},
  year={2019},
  pages={151--158}
}

@article{silver2005cooperative,
  title={Cooperative pathfinding},
  author={Silver, David},
  journal={Proceedings of the AAAI Conference on Artificial Intelligence and Interactive Digital Entertainment},
  volume={1},
  number={1},
  pages={117--122},
  year={2005}
}

@inproceedings{yu2013planning,
  title={Planning optimal paths for multiple robots on graphs},
  author={Yu, Jingjin and LaValle, Steven M},
  booktitle={2013 IEEE International Conference on Robotics and Automation},
  pages={3612--3617},
  year={2013},
  organization={IEEE}
}

@article{ma2017lifelong,
  title={Lifelong multi-agent path finding for online pickup and delivery tasks},
  author={Ma, Hang and Li, Jiaoyang and Kumar, TK and Koenig, Sven},
  journal={arXiv preprint arXiv:1705.10868},
  year={2017}
}

@inproceedings{ma2017multi,
  title={Multi-agent path finding with delay probabilities},
  author={Ma, Hang and Kumar, TK Satish and Koenig, Sven},
  booktitle=aaai,
  volume={31},
  number={1},
  year={2017}
}

@inproceedings{li2021lifelong,
  title={Lifelong multi-agent path finding in large-scale warehouses},
  author={Li, Jiaoyang and Tinka, Andrew and Kiesel, Scott and Durham, Joseph W and Kumar, TK Satish and Koenig, Sven},
  booktitle=aaai,
  volume={35},
  number={13},
  pages={11272--11281},
  year={2021}
}

@article{gandotra2025anytime,
  title={Anytime Single-Step MAPF Planning with Anytime PIBT},
  author={Gandotra, Nayesha and Veerapaneni, Rishi and Saleem, Muhammad Suhail and Harabor, Daniel and Li, Jiaoyang and Likhachev, Maxim},
  journal={arXiv preprint arXiv:2504.07841},
  year={2025}
}

@inproceedings{atzmon2018robust,
  title={Robust multi-agent path finding},
  author={Atzmon, Dor and Stern, Roni and Felner, Ariel and Wagner, Glenn and Bart{\'a}k, Roman and Zhou, Neng-Fa},
  booktitle={Proceedings of the International Symposium on Combinatorial Search},
  volume={9},
  number={1},
  pages={2--9},
  year={2018}
}

@book{lavalle2006planning,
  title={Planning algorithms},
  author={LaValle, Steven M},
  year={2006},
  publisher={Cambridge university press}
}

@inproceedings{felner2017search,
  title={Search-based optimal solvers for the multi-agent pathfinding problem: Summary and challenges},
  author={Felner, Ariel and Stern, Roni and Shimony, Solomon and Boyarski, Eli and Goldenberg, Meir and Sharon, Guni and Sturtevant, Nathan and Wagner, Glenn and Surynek, Pavel},
  booktitle={Proceedings of the International Symposium on Combinatorial Search},
  volume={8},
  number={1},
  pages={29--37},
  year={2017}
}

@inproceedings{shaoul2024accelerating,
  title={Accelerating search-based planning for multi-robot manipulation by leveraging online-generated experiences},
  author={Shaoul, Yorai and Mishani, Itamar and Likhachev, Maxim and Li, Jiaoyang},
  booktitle={Proceedings of the International Conference on Automated Planning and Scheduling},
  volume={34},
  pages={523--531},
  year={2024}
}

@article{d2012guest,
  title={Guest editorial: A revolution in the warehouse: A retrospective on kiva systems and the grand challenges ahead},
  author={D'Andrea, Raffaello},
  journal={IEEE Transactions on Automation Science and Engineering},
  volume={9},
  number={4},
  pages={638--639},
  year={2012},
  publisher={IEEE}
}

@inproceedings{wang2025lns2+,
  title={LNS2+ RL: Combining multi-agent reinforcement learning with large neighborhood search in multi-agent path finding},
  author={Wang, Yutong and Duhan, Tanishq and Li, Jiaoyang and Sartoretti, Guillaume},
  booktitle=aaai,
  volume={39},
  number={22},
  pages={23343--23350},
  year={2025}
}

@article{paul2022multi,
  title={Multi agent path finding using evolutionary game theory},
  author={Paul, Sheryl and Deshmukh, Jyotirmoy V},
  journal={arXiv preprint arXiv:2212.02010},
  year={2022}
}

@inproceedings{boyarski2015icbs,
  title={Icbs: The improved conflict-based search algorithm for multi-agent pathfinding},
  author={Boyarski, Eli and Felner, Ariel and Stern, Roni and Sharon, Guni and Betzalel, Oded and Tolpin, David and Shimony, Eyal},
  booktitle={Proceedings of the International Symposium on Combinatorial Search},
  volume={6},
  number={1},
  pages={223--225},
  year={2015}
}

@article{lam2022branch,
  title={Branch-and-cut-and-price for multi-agent path finding},
  author={Lam, Edward and Le Bodic, Pierre and Harabor, Daniel and Stuckey, Peter J},
  journal={Computers \& Operations Research},
  volume={144},
  pages={105809},
  year={2022},
  publisher={Elsevier}
}

@article{li2025fico,
  title={FICO: Finite-Horizon Closed-Loop Factorization for Unified Multi-Agent Path Finding},
  author={Li, Jiarui and Zanardi, Alessandro and Pecora, Federico and Zhang, Runyu and Zardini, Gioele},
  journal={arXiv preprint arXiv:2511.13961},
  year={2025}
}

@article{sharon2013increasing,
  title={The increasing cost tree search for optimal multi-agent pathfinding},
  author={Sharon, Guni and Stern, Roni and Goldenberg, Meir and Felner, Ariel},
  journal=ai,
  volume={195},
  pages={470--495},
  year={2013},
  publisher={Elsevier}
}

@article{mayne2000constrained,
  title={Constrained model predictive control: Stability and optimality},
  author={Mayne, David Q and Rawlings, James B and Rao, Christopher V and Scokaert, Pierre OM},
  journal={Automatica},
  volume={36},
  number={6},
  pages={789--814},
  year={2000},
  publisher={Elsevier}
}

@inproceedings{barer2014suboptimal,
  title={Suboptimal variants of the conflict-based search algorithm for the multi-agent pathfinding problem},
  author={Barer, Max and Sharon, Guni and Stern, Roni and Felner, Ariel},
  booktitle={Proceedings of the International Symposium on Combinatorial Search},
  volume={5},
  number={1},
  pages={19--27},
  year={2014}
}

@inproceedings{surynek2016empirical,
  title={An empirical comparison of the hardness of multi-agent path finding under the makespan and the sum of costs objectives},
  author={Surynek, Pavel and Felner, Ariel and Stern, Roni and Boyarski, Eli},
  booktitle={Proceedings of the International Symposium on Combinatorial Search},
  volume={7},
  number={1},
  pages={145--146},
  year={2016}
}

@article{bachtiar2016continuity,
  title={Continuity and monotonicity of the MPC value function with respect to sampling time and prediction horizon},
  author={Bachtiar, Vincent and Kerrigan, Eric C and Moase, William H and Manzie, Chris},
  journal={Automatica},
  volume={63},
  pages={330--337},
  year={2016},
  publisher={Elsevier}
}

@book{grne2013nonlinear,
  title={Nonlinear model predictive control: theory and algorithms},
  author={Grne, Lars and Pannek, Jrgen},
  year={2013},
  publisher={Springer Publishing Company, Incorporated}
}

@inproceedings{veerapaneni2025windowed,
  title={Windowed MAPF with completeness guarantees},
  author={Veerapaneni, Rishi and Saleem, Muhammad Suhail and Li, Jiaoyang and Likhachev, Maxim},
  booktitle={Proceedings of the AAAI Conference on Artificial Intelligence},
  volume={39},
  number={22},
  pages={23323--23332},
  year={2025}
}
}

\end{document}